%% file: main.tex
\pdfoutput=1
\documentclass{article}

\usepackage[main,final]{neurips_2026}

\usepackage[utf8]{inputenc} %
\usepackage[T1]{fontenc}    %
\usepackage{hyperref}       %
\usepackage{url}            %
\usepackage{booktabs}       %
\usepackage{amsfonts}       %
\usepackage{nicefrac}       %
\usepackage{microtype}      %
\usepackage{xcolor}         %
\usepackage{placeins}
\usepackage{amsmath}
\usepackage{amsthm}
\usepackage{graphicx}
\usepackage{subcaption}
\usepackage[capitalise,noabbrev]{cleveref}
\usepackage{multirow}

\newtheorem{lemma}{Lemma}
\newtheorem{proposition}{Proposition}
\newtheorem{conjecture}{Conjecture}

\title{Learned Lagrangian Models of PDEs \\via Euler--Lagrange Residual Minimization}
\author{Lyra Zhornyak \\
  University of Pennsylvania \\
  \texttt{\small zhornyak@seas.upenn.edu}  \And
  Eric Forgoston \\
  Montclair State University \\
  \texttt{\small eric.forgoston@montclair.edu}
  \And
  M. Ani Hsieh \\
  University of Pennsylvania \\
  \texttt{\small mya@seas.upenn.edu}}
\date{February 2026}

\begin{document}

\maketitle

\begin{abstract}
We present the first method to directly use a learned continuous Lagrangian to forecast the dynamics of systems governed by partial differential equations, exploiting the inherent conservative structure to achieve stable long-range predictions. 
We develop an integrator that minimizes the squared Euler--Lagrange residual via a discretization-agnostic, structure-preserving construction on local space-time patches.  
Different from integrators for analytical models, integrators for learned models should decouple model error (phase error) from integration error (conservation error).
By minimizing the residual locally on overlapping patches, we bypass the global implicit solve inherent to fixed discretizations, which slows time- and space-stepping and complicates learning. 
Our method scales linearly with domain size via Jacobi iteration, and places no structural requirements on the learned network, allowing it to be coupled with existing physics-guided machine learning (ML) methods.
We validate our approach on a learned representation of a double pendulum, a one-dimensional wave equation, and a two-dimensional wave equation.  
Our method achieves long-horizon predictions with error comparable to classical symplectic methods while generalizing to spatially varying dynamics and arbitrary boundary conditions without retraining.
\end{abstract}

\input{intro}
\input{related}
\input{methods}

\input{experiments}
\input{conclusion}

\newpage
\begin{ack}
We would like to thank Benjamin Shaffer and all members of the ScalAR lab for their thoughtful comments and discussions during the development of this work. Additionally, we thank our anonymous reviewers for their careful review of the initial draft, which helped sharpen our arguments.
This work was supported by ONR Award N000142512171. The authors do not have any competing interests.
\end{ack}

\FloatBarrier
\bibliography{references}

\FloatBarrier
\newpage
\input{appendix}

\end{document}

%% file: intro.tex
\section{Introduction}
Dynamical systems are mathematical models that describe how a system evolves over time, typically specified by a state space and a set of rules governing how the states change over time and space.  
Estimating dynamical systems models from data is challenging because the dynamics are often nonlinear and high-dimensional, making it difficult to identify both the model structure and parameters from noisy, sparse, or partially observed time series measurements of system states.  
Nevertheless, modern approaches like physics-informed/guided or structure-preserving machine learning have significantly improved our ability to model and capture complex dynamics from data~\citep{thuerey2021physics}.  
In recent years, neural networks have successfully been used to learn Hamiltonians~\citep{greydanus2019hamiltonian}, Lagrangians~\citep{cranmer2020lagrangian, lutter2019deep}, and dynamics directly from data~\citep{sanchez2020learning, chen2018neural}. 
Yet while these models capture the underlying physics with increasing accuracy, strikingly little attention is paid to how these models are \emph{used} for tasks where forward integration plays a central role, \emph{e.g.}, forward simulation, prediction and forecasting, estimation and filtering.

This oversight matters because forecasting a system's behavior places different demands on the learned model than finding a solution does.
While these two uses are closely related --- knowing how to produce a good local solution does help forecasting --- their objectives are distinct: \emph{solving} demands local accuracy and sensitivity to the equations, while \emph{forecasting} demands stability and bounded long-term error. 
For systems that conserve energy (waves, vibrations, orbits, etc.), this distinction is particularly important as even small per-step errors compound over long simulations, eventually producing solutions that violate conservation laws and become physically meaningless.

Existing work has predominantly focused on producing better and better models of the dynamics, implicitly assuming that forecasting errors can be overcome by either more training or specialized architectures.
This assumption fails most severely for chaotic systems and systems defined by partial differential equations (PDEs).
Accumulated errors can violate the structure of the system, pushing a trajectory outside of the stable, learned regime, thereby causing a catastrophic failure in the prediction.
Classical numerical analysis has long known that preserving structure can matter as much as per-step accuracy.
For Hamiltonian and Lagrangian systems, symplectic integrators trade local fidelity for energy conservation and long-term stability~\citep{reich1999backward, marsden2001discrete}.
Despite this, most learned-physics methods treat the integrator as an implementation detail.

Motivated by this gap, we introduce Euler--Lagrange Minimization (ELM), the first integrator designed on structure-preserving principles for learned continuous Lagrangian models of PDEs. 
Prior work on learned variational integrators has been limited to finite-dimensional systems~\citep{saemundsson2020variational}.
ELM extends to multidimensional field theories by evaluating the Euler--Lagrange residual on Hermite-interpolated space-time patches and minimizing the residual locally via Newton iteration.
The local construction maintains linear scaling with domain size by avoiding a coupled, global solution process, and the minimization-based framing naturally incorporates boundary conditions through constraints on the optimized variables. 
When the residual is driven to zero, the integrator is symplectic for Lagrangians with a constant mass matrix, inheriting the long-term energy conservation of classical variational integrators.
Beyond sufficient smoothness of the learned Lagrangian, ELM imposes no structural constraints on the model used to learn it, facilitating use with domain-specific physics-guided approaches.

We validate and demonstrate the performance of ELM on three systems of increasing dynamical complexity: a chaotic double pendulum, where ELM achieves energy drift comparable to state-of-the-art symplectic methods; the one-dimensional (1D) wave equation, where a learned Lagrangian density outperforms Lagrangian Neural Networks~\citep{cranmer2020lagrangian} in both accuracy and compute time; and the two-dimensional (2D) wave equation, where (to the best of our knowledge) ELM provides the first forward simulation of a 2D PDE using a learned continuous Lagrangian density.
By not preemptively discretizing space, ELM allows models trained on homogeneous domains to produce, without modification, physically correct reflection, transmission, and interference effects for spatially varying dynamics and new boundary conditions.

We summarize related work in \cref{sec:related} and describe our method and conditions for symplecticity in \Cref{sec:methods}. \Cref{sec:results} demonstrates the low long-horizon energy drift of ELM on all three systems and we conclude with a discussion on limitations and future work in \Cref{sec:conclusion}. Code is \href{https://github.com/lzhornyak/Euler-Lagrange-Minimization}{publicly released}.

%% file: related.tex
\section{Related Work}\label{sec:related}
\paragraph{Operator Methods for Differential Equations}
The most conceptually straightforward method for learning a dynamical system might be directly modeling the underlying equations.
Physics-Informed Neural Networks \citep{raissi2019physics}, for example, add the residual of the governing equation applied to the network outputs as a loss term. 
However, this approach often requires strong prior knowledge on the function being learned to ensure training is successful \citep{krishnapriyan2021characterizing}.
To bypass these limitations and to enable more general models, one can alternatively learn a map to an operator space where solutions are more easily found.
Some examples include DeepONet \citep{lu2021deeponet}, which uses a branch-trunk architecture to map the input and output to a common space; Fourier Neural Operators \citep{li2021fourier}, which learn operators using truncated Fourier modes to achieve resolution invariance; and other physics-informed operator learning variants such as PI-DeepONet \citep{wang2021pideeponet} and PINO \citep{li2024pino}.
Operator-based methods rely on a strict definition of the domain and are highly structured.
To learn local physical laws and allow for arbitrary domains, graph neural network approaches such as Graph Network Simulators \citep{sanchez2020learning} and MeshGraphNets \citep{pfaff2021learning} have successfully been shown to represent classical numerical schemes such as finite-differences \citep{brandstetter2022message}.
While these methods are effective on a wide variety of problems, none of them encode or preserve the conservation laws of the systems they model.
As such, even small per-step errors in integration will accumulate and eventually violate fundamental conservation laws.
Post-hoc remedies such as filtering can suppress some high-frequency error \citep{lippe2023pderefiner}, but do not address the structural cause of the drift.
This limitation is not just architectural; these methods implicitly fuse the dynamics with their integrator, motivating efforts to instead learn the variational structure from which the dynamics arise.

\begin{figure}[t]
    \centering
    \includegraphics[width=\linewidth]{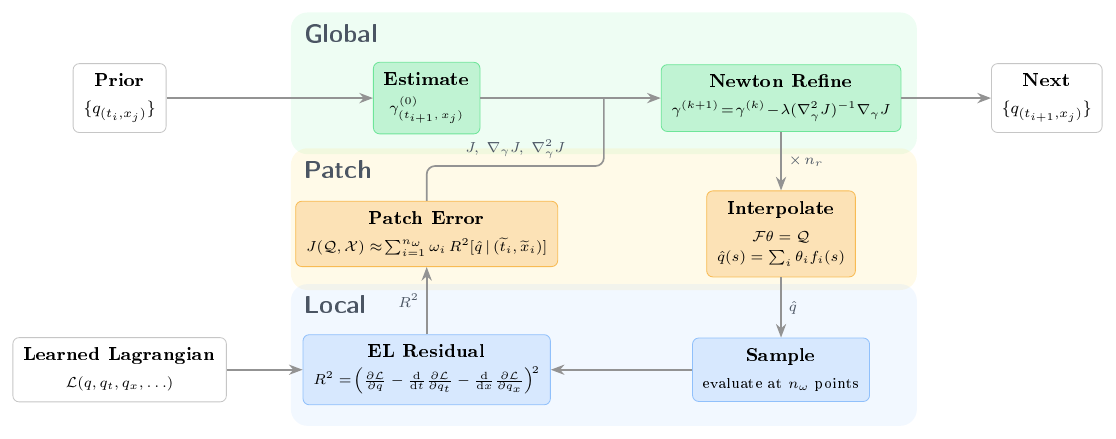}
    \caption{The construction of ELM at three scales.
    \textit{Local}: the squared Euler--Lagrange residual \(R^2\) is computed from the learned Lagrangian \(\mathcal{L}\) and evaluated at \(n_\omega\) quadrature points.
    \textit{Patch}: Hermite interpolation fits a set of basis functions to known node states to form \(\hat{q}\). The resulting residuals are summed to form the patch error \(J\).
    \textit{Global}: an initial estimate \(\gamma^{(0)}\) is refined by damped Newton iteration using \(J\); iterating \(n_r\) times around the loop advances the solution by one time step.}
    \label{fig:architecture}
\end{figure}

\paragraph{Structure-Preserving Learning}
Variational mechanics, by contrast, considers the whole trajectory rather than any individual point.
Conserved quantities thus manifest as structural symmetries of the governing functions, \emph{e.g.} a Lagrangian or a Hamiltonian.
These functions can either be learned directly, as in Hamiltonian \citep{greydanus2019hamiltonian} and Lagrangian Neural Networks (LNN) \citep{cranmer2020lagrangian}, or used as the basis for designing the structure of the network, as in DeLaN \citep{lutter2019deep} and constrained formulations \citep{finzi2020simplifying} (see \citet{yu2024learning} for a broader survey).
In most cases, the method to produce long-horizon forecasts, \emph{i.e.} the integrator, is barely considered \citep{chen2020symplectic}.
Existing works that explicitly preserve structure apply only to systems governed by ODEs \citep{saemundsson2020variational} or rely on fixed spatial discretizations.
\citet{offen2024learning} learn a discrete Lagrangian for a prescribed discretization which leads to an implicit update scheme that may couple all nodes into a single global solve.
This is computationally expensive and can be numerically unstable, limiting this method to only 1D PDEs.
By assuming that the grid is regular, FS-HNN \citep{li2026frequency} extends structure-preserving approaches to two dimensions by combining a learned Hamiltonian functional (via DeepONet) and a learned convolutional kernel.
FS-HNN is not symplectic, however, and achieves energy conservation by projecting the operator onto a skew-symmetric subspace.
With ELM, we instead use learned continuous Lagrangian densities independent of the boundary conditions and discretization.
The structure-preserving construction improves energy conservation by orders of magnitude over long time horizons.
 
\paragraph{Classical Variational Integrators}
Variational integrators, derived directly from the discrete Euler--Lagrange equations, are symplectic by construction \citep{marsden2001discrete}.
For PDEs that arise from field theories, this construction extends to multisymplectic integrators that preserve the underlying geometric structure \citep{marsden1998multisymplectic}. 
The long-time numerical stability of these multisymplectic integrators has been proven over exponential time scales \citep{reich1999backward,hairer2006geometric} where the computed energy of the system oscillates around its true value rather than accumulating error.
The accuracy of a variational integrator depends on how well the discrete Lagrangian approximates the continuous one.
\citet{leok2012general} and \citet{hall2015spectral} demonstrate a Galerkin approach, approximating trajectories with a finite-dimensional function space, that can achieve arbitrarily high integration order.
In contrast to these discrete constructions, our ELM method can be understood as a least-squares collocation of the strong-form Euler--Lagrange equation, specifically designed for learned Lagrangian field theories.
Whereas classical variational integrators solve the discrete Euler--Lagrange equations of a discrete action constructed from a known Lagrangian, ELM minimizes the squared Euler--Lagrange residual on space-time patches using a learned Lagrangian, evaluated at a set of quadrature points.
When this residual is zero, the integrator inherits the symplectic guarantees of classical theory (\cref{sec:near_symp}).
In practice, we oversample the quadrature (using more points than required for symplecticity), trading exact symplecticity for stability of the patch solver and robustness to imperfections in the learned Lagrangian.
 

%% file: methods.tex
\section{Methodology}\label{sec:methods}

Given a configuration manifold \(Q\) and coordinate space \(X \subseteq \mathbb{R}^{n+1}\), suppose we have a learned map from the first jet bundle \(\mathcal{L} : J^1(X,Q) \rightarrow \mathbb{R}\).
We want to find the path \(q : X \rightarrow Q\) such that the action
\begin{equation}
    S[q] = \int_X \mathcal{L}(q(\xi), q_{\xi_0}(\xi), \dots, q_{\xi_n}(\xi)) \;\mathrm{d}X \label{eq:action}
\end{equation}
is stationary, where \(\xi = (\xi_0, ~\dots~, \xi_{n}) \in X\) and \(q_{\xi_i} = \partial q / \partial \xi_i\).
In Lagrangian mechanics, \(\mathcal{L}\) is the Lagrangian density and \(S\) is the action.
As an example, for the 1D wave equation, \(\mathcal{L}(q, q_t, q_x) = \frac{1}{2}\mu\, q_t^2 - \frac{1}{2}c^2 q_x^2\). 
For simplicity, this methodology is formulated assuming a PDE with one spatial and one temporal dimension (\((t,x) \in X=\mathbb{R}^2\)). 
However, the formulation for ODEs and higher-dimensional PDEs proceeds in an identical manner.

Standard variational methods discretize the action \(S[q]\) to form an integrator.
This requires symbolic manipulation of \(\mathcal{L}\) and so is not possible for learned Lagrangians.
ELM instead represents a new subcategory of integrator: a collocation of the strong-form Euler--Lagrange equation.
This allows a learned Lagrangian to be simulated directly, without converting to momentum space or inverting the Hessian of \(\mathcal{L}\) with respect to the velocities.
While ELM can equally simulate an analytical Lagrangian, its benefits are specific to learned, black-box Lagrangians.
ELM is a minimization-based integrator designed for learned models, constructed at three scales (\cref{fig:architecture}).
A local, pointwise residual measuring the deviation from the learned dynamics (\cref{sec:local error}) defines the error induced by a path approximated from a finite set of nodes (\cref{sec:patch loss}).
Starting from an initial guess, a damped Newton solve swept via Jacobi iteration across overlapping patches advances the global state without a coupled solve (\cref{sec:global min}).
Constructed on variational principles, this integrator is symplectic (to quadrature error) when the residual vanishes and structure-preserving in practice (\cref{sec:near_symp}).

\subsection{Local Residual} \label{sec:local error}

Suppose we have some variation \(\delta q\) from the true path \(q\) such that \(\delta q = 0\) on the boundary of the integration domain. 
Since \(S[q]\) is a functional, a necessary condition for a minimum to exist is when its first variational derivative, \(\delta S\), equals zero:
\begin{equation*}
    \delta S = \iint \left( \frac{\partial\mathcal{L}}{\partial q} - \frac{\mathrm{d}}{\mathrm{d}t} \frac{\partial\mathcal{L}}{\partial q_t} - \frac{\mathrm{d}}{\mathrm{d}x} \frac{\partial\mathcal{L}}{\partial q_x}\right) \delta q \;\mathrm{d}t\;\mathrm{d}x = 0.
\end{equation*}
Since \(\delta q\) can be any function, by the Fundamental Lemma of the calculus of variations, \(\delta S = 0\) only when the quantity in the brackets is zero everywhere on \(X\).
We denote this term as \(R\):
\begin{equation}
R[q~|~(t,x)] \equiv \left( \frac{\partial\mathcal{L}}{\partial q} - \frac{\mathrm{d}}{\mathrm{d}t} \frac{\partial\mathcal{L}}{\partial q_t} - \frac{\mathrm{d}}{\mathrm{d}x} \frac{\partial\mathcal{L}}{\partial q_x}\right)~\!\Bigg|_{(t,x)}.  \label{eq:local_residual}
\end{equation}
\(R=0\) is the Euler--Lagrange equation.
Given an approximate path \(\hat{q}\), the squared Euler--Lagrange residual \(R^2[\hat{q}~|~(t,x)]\) is a measure of the deviation from the dynamics represented by \(\mathcal{L}\).

\subsection{Patch Error} \label{sec:patch loss}

Suppose we have a set of \(m\) coordinates and a set of corresponding field values and their derivatives:
\begin{gather*}
    \mathcal{X} = \begin{bmatrix} ~ (t_1,x_1) & \dots & (t_m, x_m) ~\end{bmatrix} \in X^m, \\
    \mathcal{Q} = \begin{bmatrix} ~q(\mathcal{X}) & q_t(\mathcal{X}) & q_x(\mathcal{X}) & q_{tx}(\mathcal{X}) ~\end{bmatrix}^T \in \mathbb{R}^{4m}.
\end{gather*}
Here \(q(\mathcal{X}) = \begin{bmatrix} ~q(t_1,x_1) & \dots & q(t_m,x_m)~ \end{bmatrix}\) denotes the \(m\) values taken by \(q\) at the coordinates \(\mathcal{X}\).
Together, \(\mathcal{X}\) and \(\mathcal{Q}\) define a patch.
For a set of \(4m\) functions \(\mathcal{F} = \{f_1, \dots, f_{4m}\}\), \(f_i : X \rightarrow \mathbb{R}\), chosen so that \cref{eq:basis_fit} is nonsingular, we construct a Hermite interpolation \(\hat{q}\) by fitting \(\mathcal{F}\) to this patch.
The interpolation coefficients \(\theta \in \mathbb{R}^{4m}\) are found by solving
\begin{equation}
    \begin{bmatrix} ~f_1(\mathcal{X}) & \frac{\partial f_1(\mathcal{X})}{\partial t} & \frac{\partial f_1(\mathcal{X})}{\partial x} & \frac{\partial^2 f_1(\mathcal{X})}{\partial t \partial x} \\ \vdots & \vdots & \vdots & \vdots \\ f_{4m}(\mathcal{X}) & \frac{\partial f_{4m}(\mathcal{X})}{\partial t} & \frac{\partial f_{4m}(\mathcal{X})}{\partial x} & \frac{\partial^2 f_{4m}(\mathcal{X})}{\partial t \partial x}~ \end{bmatrix}^T ~\theta = \mathcal{Q}. \label{eq:basis_fit}
\end{equation}
The resulting best fit function is denoted as \(\hat{q} = \sum_{i =1}^{4m} \theta_i f_i\), a map from \(X\) to \(\mathbb{R}\).

We define the residual accumulated by this interpolation as the patch error:
\begin{equation}
    J(\mathcal{Q}, \mathcal{X}) \equiv \iint_{\mathrm{Conv}(\mathcal{X})} R^2[\hat{q}~|~(t,x)] \;\mathrm{d}t~\mathrm{d}x \approx \sum_{i = 1}^{n_\omega} \omega_iR^2\left[\hat{q}~\middle|~\left(\widetilde{t}_i,\widetilde{x}_i\right)\right], \label{eq:lagrangian loss}
\end{equation}
where \(\mathrm{Conv}(\mathcal{X})\) is the convex hull of \(\mathcal{X}\), from which we select \(n_\omega\) points \((\widetilde{t},\widetilde{x})\) and associated weights \(\omega\).
Throughout this work, we take \(\mathcal{F}\) to be a monomial basis in both coordinates (e.g. \(\{t^a x^b \mid a \le 3,\ b \le 5\}\)) for simplicity.
Since \(\mathcal{F}\) is the tensor-product monomial basis, \cref{eq:basis_fit} is exactly two classical one-dimensional Hermite interpolations, guaranteeing the existence and uniqueness of its solution.

\subsection{Global Integrator} \label{sec:global min}

To advance the solution by one time step, we seek the node \(\gamma \in J^1(X,Q)\) at a new coordinate \(\eta \in X\) that minimizes the patch error given the known states. This defines the integrator \(G\):
\begin{equation}
    G(\mathcal{Q}, \mathcal{X}, \eta) \equiv \arg\min_{\gamma}\; J\!\left([\mathcal{Q}; \gamma],\; [\mathcal{X}; \eta]\right).\label{eq:integrator}
\end{equation}

While this minimization can be implemented in a variety of ways, we use damped Newton iteration for simplicity (\cref{app:cost} demonstrates that simple gradient descent is insufficient).
Given an initial guess \(\gamma^{(0)}\) extrapolated from previous time steps, the update is
\begin{equation}
    \gamma^{(k+1)} = \gamma^{(k)} - \lambda \left(\nabla^2_\gamma \left(J(\gamma^{(k)})\right)\right)^{-1} \nabla_\gamma \left(J(\gamma^{(k)})\right), \label{eq:newton}
\end{equation}
where \(\lambda \in (0, 1]\) is a damping parameter.
The choice of damping parameter is highly dependent on the patch overlap and the dynamical coupling of the system; for ODEs, \(\lambda=1\) is typically fine, whereas for 2D PDEs, \(\lambda=0.5\) is often needed.
The number of required iterations, by contrast, is mostly dependent on the behavior of the system near the minimum and the selection of \(\lambda\).

\subsection{\texorpdfstring{Structure and Symplecticity}{Structure and Symplecticity}}\label{sec:near_symp}
Given a system whose dynamics are governed by an ODE and whose Lagrangian is given by \(\mathcal{L}(q,q_t)\), we employ the Gauss--Legendre (GL) quadrature on \(n_\omega\) points, which exactly integrates polynomials up to degree \(2n_\omega - 1\). We define the discrete action on the patch \([t_0, t_1]\) as
\begin{equation*}
    S_d(q(t_0), q_t(t_0), q(t_1), q_t(t_1)) \equiv \sum_{i=1}^{n_\omega} \omega_i\, \mathcal{L}(\hat{q}(\widetilde{t}_i), \hat{q}_t(\widetilde{t}_i)).
\end{equation*}
Unlike classical discrete methods that advance \(q \in Q\), ELM advances a state \((q, q_t)\) on the tangent bundle \(TQ\).
From this state, the momentum is given by \(p=\frac{\partial \mathcal{L}}{\partial q_t}\) and the differential of the discrete action can be expressed as
\begin{equation}
  \mathrm{d}S_d = p(t_1)\,\mathrm{d}q(t_1) - p(t_0)\,\mathrm{d}q(t_0) + \sum_{i=1}^{n_\omega} \omega_i\, R[\hat{q}|\widetilde{t}_i]\,\mathrm{d}\hat{q}(\widetilde{t}_i) + \mathcal{O}(\Delta t^{2n_\omega + 1}). \label{eq:gen_func}
\end{equation}
We refer the interested reader to \cref{app:gen_func} for the full derivation.

Assume that the Newton iteration in \cref{eq:newton} converges such that \(J=0\).
Since \(J\) is the sum of the squared residuals and \(R^2 \ge 0\), this requires that \(R=0\) at each GL quadrature point and thus 
\cref{eq:gen_func} reduces to \(\mathrm{d}S_d =p(t_1)\,\mathrm{d}q(t_1) - p(t_0)\,\mathrm{d}q(t_0)\).
Under these conditions, the corresponding update can be represented by a type-1 generating function \citep{marsden2001discrete} to within GL quadrature error.
In \cref{sec:exp_ode}, we verify that, with the appropriate choice of quadrature, \(J\) can be driven to within machine precision of 0. %
Specifically, \(n_\omega=2\) is chosen for ODEs, resulting in an integrator with order \(2n_\omega = 4\) matching the order of the quadrature.
For Lagrangians with a constant mass matrix, the quadrature error vanishes exactly to result in a symplectic method:
\begin{proposition}\label{prop:mechanical}
Let \(\mathcal{L}(q, q_t)\) be \(C^2\) with nonsingular Legendre transform, and consider one ELM step in which the Newton iteration converges to \(J = 0\).
If \(\mathcal{L} = \tfrac12 q_t^\top M q_t - V(q)\) with \(M\) constant and symmetric, the resulting map preserves the canonical two-form and is thus symplectic.
For a general \(\mathcal{L}\), the resulting map is symplectic to within Gauss--Legendre quadrature error.
\end{proposition}
The proof is given in \cref{app:wronskian}; the general case follows from \cref{app:gen_func}.

In principle, the argument on symplecticity above could be extended to PDEs.
In this work, however, we use a larger \(n_\omega\) to keep the Jacobi iteration stable.
This oversampling prevents \(J\) from reaching \(0\) except in edge cases where \(q \in \mathrm{span}\,\mathcal{F}\).
Since ELM is a minimization-based integrator with a symplectic construction, we refer to ELM as being structure-preserving.
For PDEs, we conjecture the following:
\begin{conjecture}\label{conj:pde}
(i) When ELM converges such that \(J = 0\) for a PDE with a constant mass, the resulting step is symplectic.
(ii) When ELM converges such that \(J > 0\), the resulting symplectic defect is governed primarily by the residual \(R\).
\end{conjecture}
Numerical evidence for both parts is given in \cref{app:wronskian}.
We leave the bounding of the convergence of the Newton iteration and the backward error analysis for future work.

%% file: experiments.tex
\section{Experiments}\label{sec:results}

\subsection{Double Pendulum}\label{sec:exp_ode}
We first verify that ELM's structure for ODEs holds in practice.
Consider the chaotic double pendulum, governed by a nonlinear ODE.
Its coupled dynamics and complicated Lagrangian make it a standard benchmark for long-horizon rollouts.
\Cref{fig:double_pendulum} shows the relative energy error over the full trajectory on a log-log scale.
We compare ELM against four standard integrators using an LNN-style estimate for \(q_{tt}\) with the same underlying learned model of the system.
The non-symplectic methods (the classical Runge--Kutta (RK4) and Dormand--Prince) result in continual and steady energy accumulation.
In contrast, symplectic methods conserve a discrete Hamiltonian, so while they do oscillate, they remain bounded over long time-scales.
We see in the figure that ELM matches the fourth-order Gauss--Legendre Runge--Kutta method (GLRK, \citealp{hairer2006geometric}) since ELM is also a fourth-order method (when optimized over two GL points).

\begin{figure}[t]
    \centering
    \includegraphics[width=0.9\linewidth]{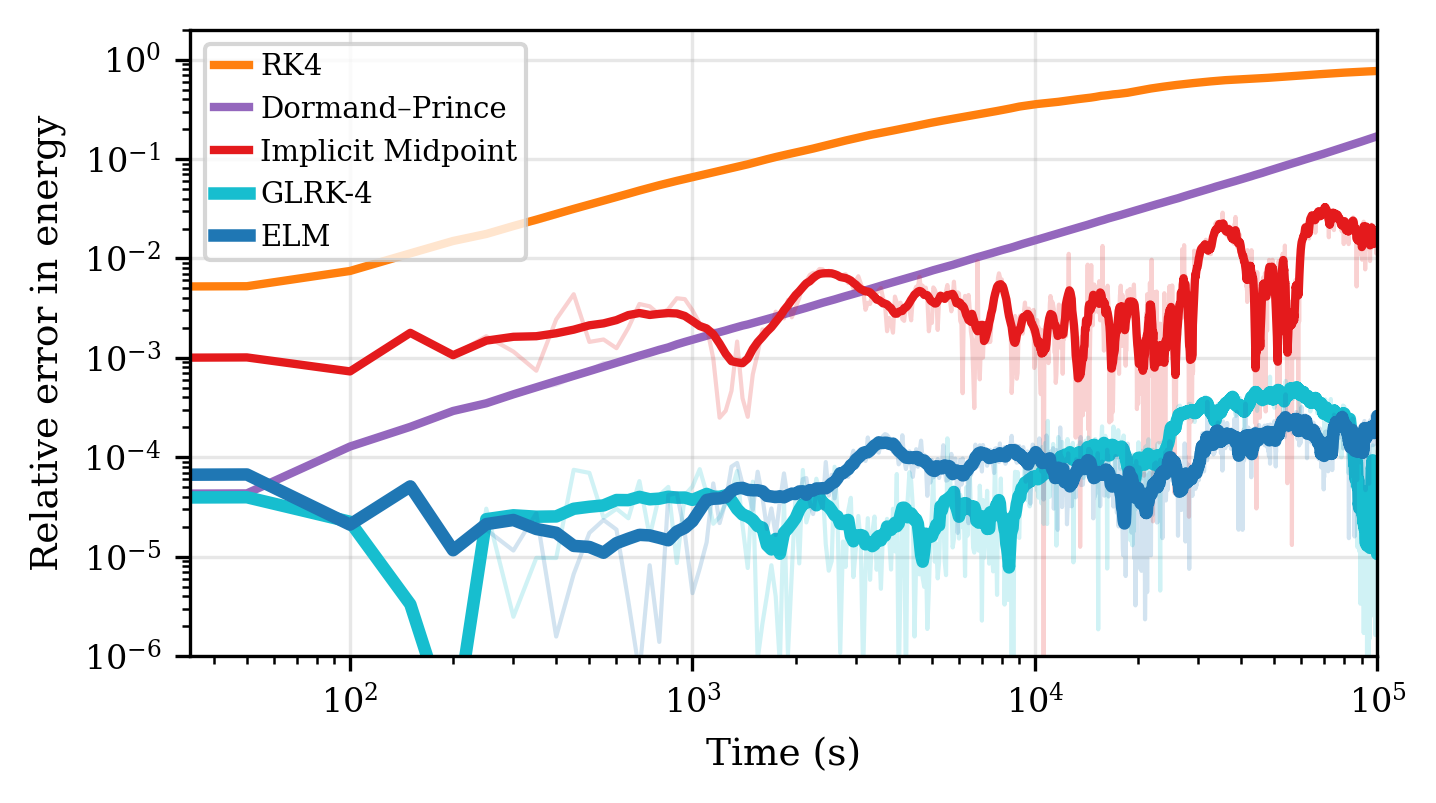}
    \caption{Relative error in energy \(|E_t-E_0|/E_0\) over \(100{,}000\,\)s for a double pendulum with initial angles \(q_1=\frac{3\pi}{7}\) and \(q_2=\frac{3\pi}{4}\).
    The double pendulum has an analytical Lagrangian \(\mathcal{L} = q_{1,t}^2 + \tfrac{1}{2}q_{2,t}^2 + q_{1,t}\,q_{2,t}\cos(q_1 - q_2) + 2\cos q_1 + \cos q_2\).
    Results are shown for ELM and for four classical integrators using an LNN-style estimate for the acceleration.
    All methods share the same learned Lagrangian.
    Faint lines are the raw values while bold lines show a moving average (11-point window) to remove the oscillations characteristic of symplectic integrators.}
    \label{fig:double_pendulum}
\end{figure}

\begin{figure}[t]
    \centering
    \begin{subfigure}{\linewidth}
        \centering
        \includegraphics[width=0.85\linewidth]{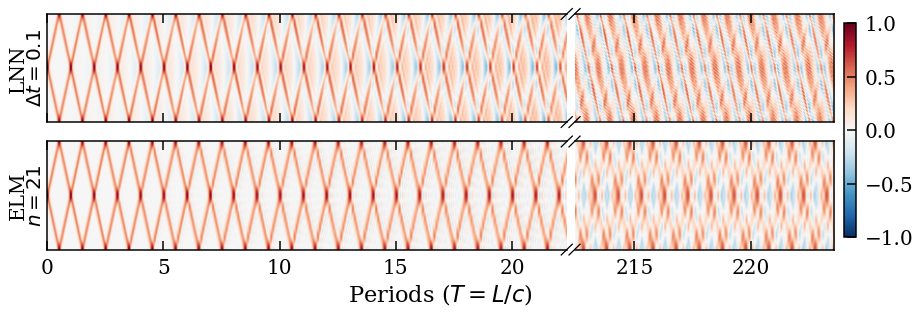}
        \caption{Spatio-temporal heatmaps of wave amplitude, showing the first \({\sim}20\) and the last \({\sim}10\) periods.}
        \label{fig:wave_heatmaps}
    \end{subfigure}

    \vspace{0.5em}
    \includegraphics[width=\linewidth]{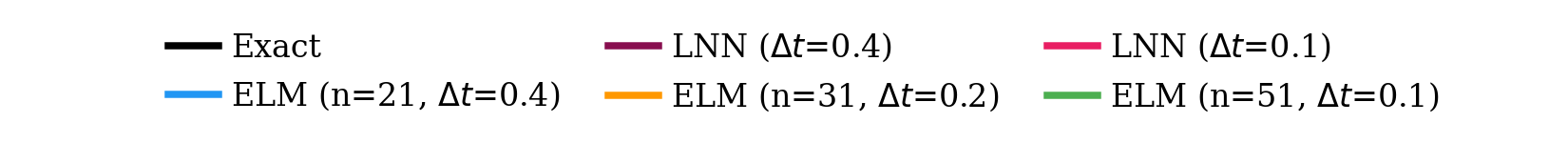}\\[-2pt]
    \begin{subfigure}[t]{0.48\linewidth}
        \centering
        \includegraphics[width=0.92\linewidth]{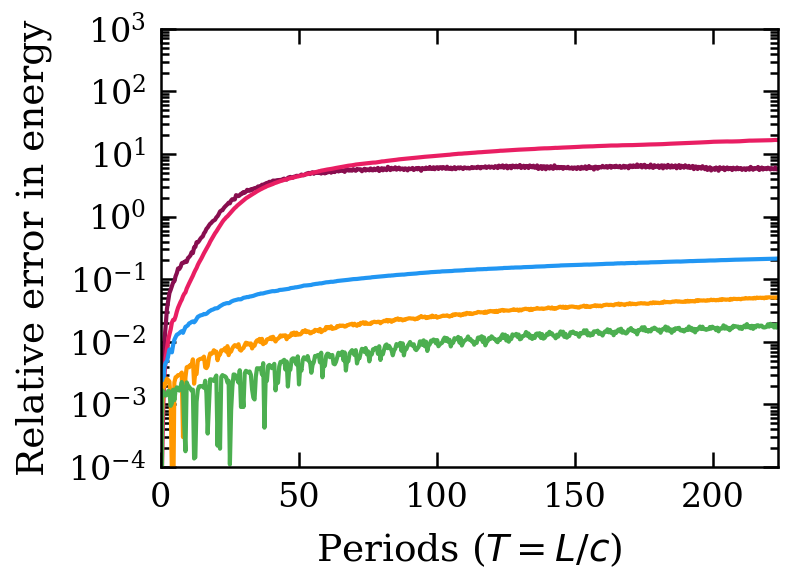}
        \caption{Relative error in energy \(|E_t - E_0|/E_0\).}
        \label{fig:wave_drift}
    \end{subfigure}
    \hfill
    \begin{subfigure}[t]{0.48\linewidth}
        \centering
        \includegraphics[width=0.92\linewidth]{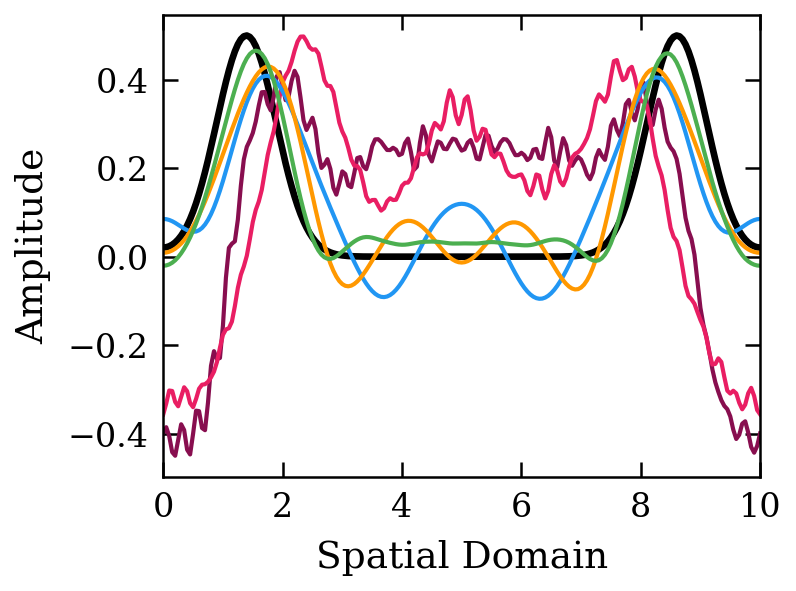}
        \caption{Spatial profile \(q(x)\) at \(t = 1000\)\,s.}
        \label{fig:wave_snapshot}
    \end{subfigure}
    \caption{(a) Wave amplitude heatmaps, (b) relative error in energy, and (c) spatial profile for the 1D wave equation \(q_{tt} = 0.05\,q_{xx}\), with periodic boundary conditions, and a Gaussian initial condition. 
    LNN uses a velocity-Verlet integrator with two time-step sizes. 
    ELM is shown at three spatio-temporal resolutions. Parameters for each method and training procedure are listed in \cref{app:cost,app:training}.
    }
    \label{fig:wave_combined}
\end{figure}

\begin{figure}[p]
    \centering
    \includegraphics[width=\linewidth]{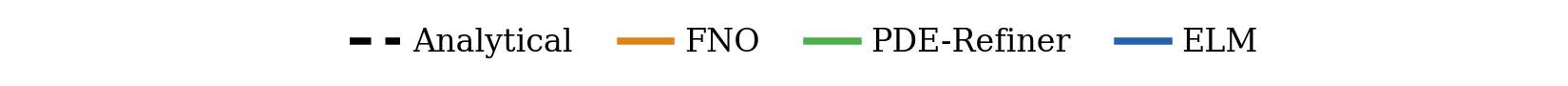}\\[-2pt]
    \begin{subfigure}[t]{0.48\linewidth}
        \includegraphics[width=\linewidth]{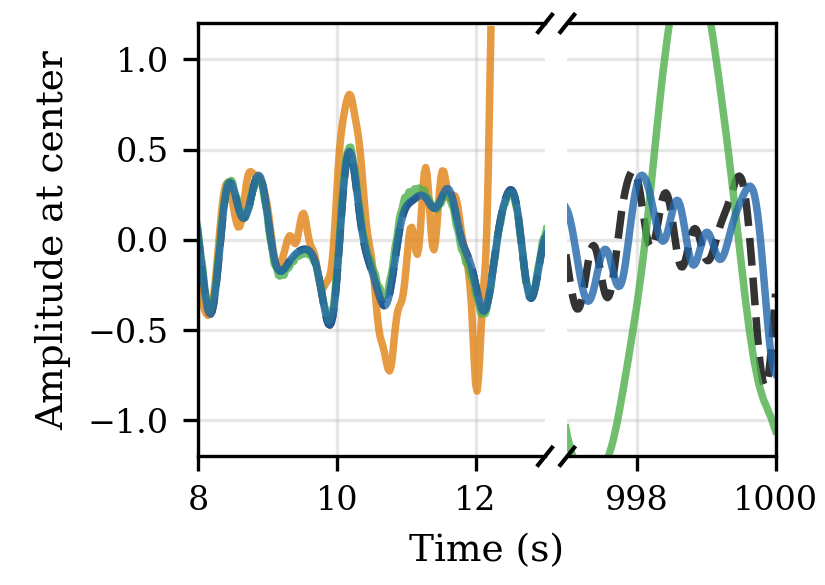}
        \caption{Center-node displacement \(q(0,0,t)\)  showing \({\sim}5\)\,s before FNO diverges and the last \({\sim}3\)\,s simulated.}
        \label{fig:wave_2d_center}
    \end{subfigure}
    \hfill
    \begin{subfigure}[t]{0.48\linewidth}
        \includegraphics[width=\linewidth]{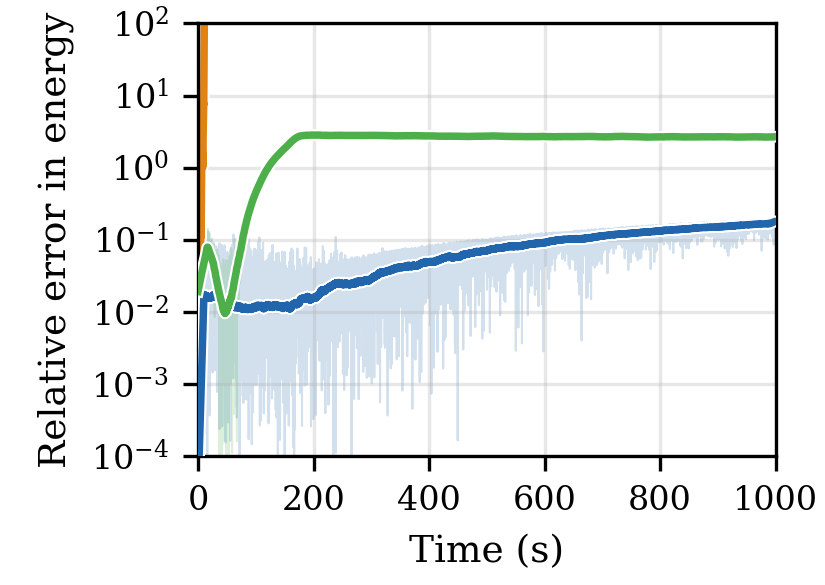}
        \caption{Relative energy error \(|E_t - E_0|/E_0\). Faded lines are raw values; bold lines show a moving average over a 10\,s window.}
        \label{fig:wave_2d_energy}
    \end{subfigure}

    \vspace{0.5em}
    \begin{subfigure}{\linewidth}
        \centering
        \includegraphics[width=\linewidth]{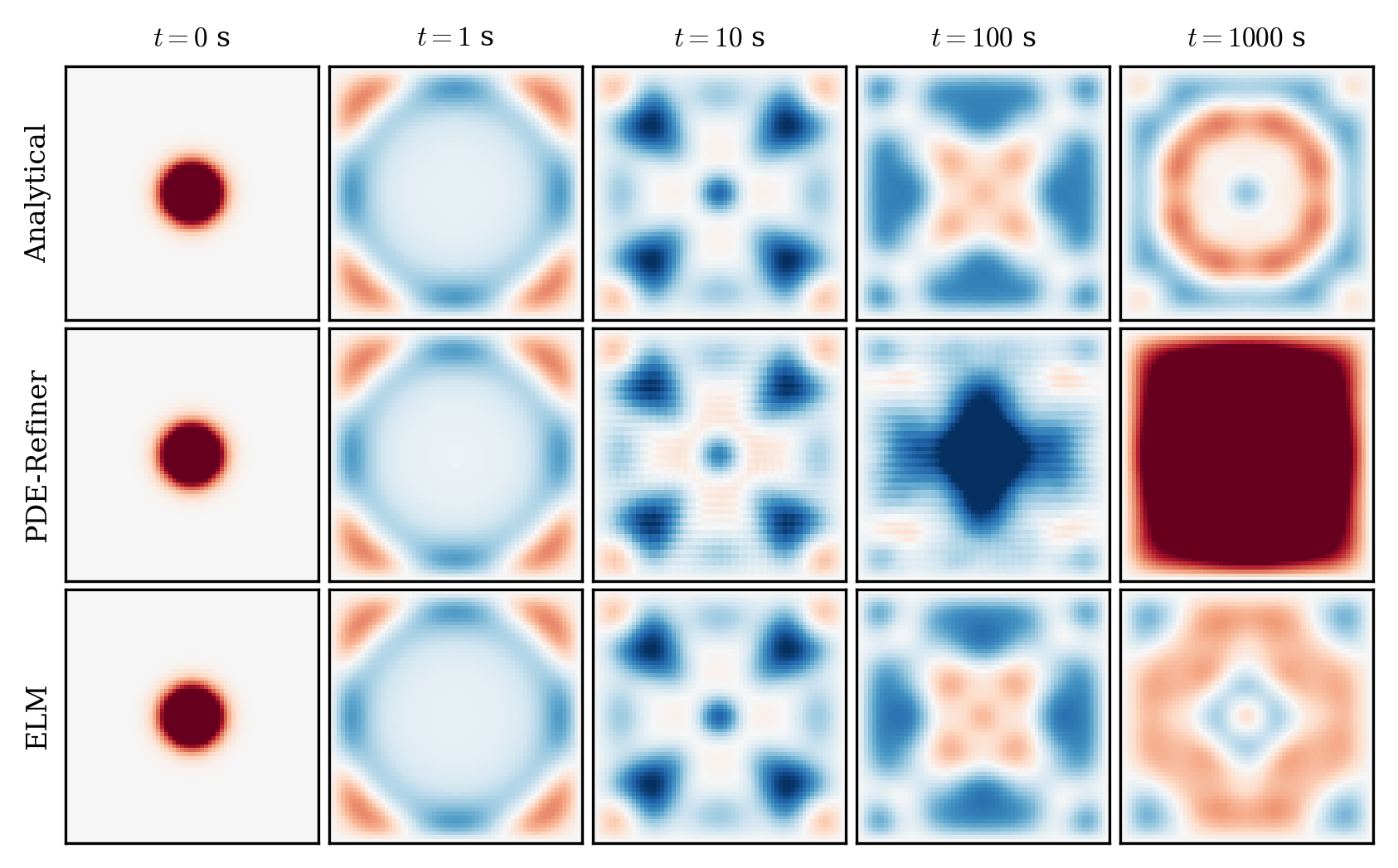}
        \caption{Spatial snapshots \(q(x,y)\) at \(t = 0, 1, 10, 100, 1000\)\,s. Top: eigenmode ground truth. Middle: PDE-Refiner. Bottom: ELM. FNO is omitted as it diverges at \(t \approx 22\)\,s.}
        \label{fig:wave_2d_snapshots_body}
    \end{subfigure}
    \caption{(a) Center amplitude, (b) relative energy error, and (c) spatial snapshots in time for the 2D wave equation \(q_{tt} = q_{xx} + q_{yy}\), using a 31 \(\times\) 31 grid with Gaussian initial condition, and Dirichlet boundary conditions over \(1{,}000\)\,s (\({\sim}350\) periods of the fundamental mode). 
    ELM uses a learned Lagrangian density; FNO and PDE-Refiner are autoregressive baselines on a 64 \(\times\) 64 grid. The eigenmode ground truth is found using a 50 \(\times\) 50 mode decomposition. Parameters for each method and training procedure are listed in \cref{app:cost,app:training}.}
    \label{fig:wave_2d_300s}
\end{figure}

\subsection{1D Wave Equation}\label{sec:exp_wave}

Having shown the symplectic underpinnings of ELM for ODEs, we now demonstrate the strength of ELM as a discretization-agnostic, structure-preserving integrator that allows learned Lagrangian density models to be applied to PDEs for the first time.
We compare the LNN wave model \citep{cranmer2020lagrangian} against ELM on a 1D wave equation, both using learned Lagrangians of equivalent Multilayer Perceptron (MLP) architectures.
LNN learns a discrete Lagrangian on a fixed discretization and integrates all nodes as a single ODE.  
Different from LNN, ELM learns the continuous Lagrangian density and integrates the learned Lagrangian using the methodology described in \cref{sec:methods}.

Over \(10{,}000\)\,s, the behavior of the two methods sharply diverges (\cref{fig:wave_combined}).
Even at the coarsest settings, using only 21 spatial nodes, \cref{fig:wave_heatmaps} shows how ELM maintains discernible and symmetric wavefronts over longer periods while LNN devolves sooner.
\Cref{fig:wave_drift} shows several orders of magnitude improvement in energy conservation, and the snapshots in \cref{fig:wave_snapshot} demonstrate the increased spatial fidelity of ELM.
Both improve as the number of nodes increases and time step size decreases.
In contrast, the fixed error induced by the learned discretization in LNN leads to diminishing returns when using smaller step sizes.
This highlights a major advantage of learning a continuous Lagrangian: we separate model errors from integration errors.
Since \emph{every} well-conditioned learned Lagrangian is a valid Lagrangian and represents some dynamics, the error in the learned model can only contribute to phase error while all energy drift is attributable to the integrator.
Some additional results and discussion are provided in \cref{app:heatmaps}, including heatmaps for the exact solution and additional integrators, plus the relative \(L^2\) error over the full \(10{,}000\)\,s. \cref{app:nonlinear} demonstrates that ELM's separation of model and integration error extends to certain classes of nonlinear PDEs.

\subsection{2D Wave Equation}\label{sec:exp_2d}
To enable ELM to integrate systems of higher dimensions, only mechanical changes are needed: the basis functions must incorporate the new coordinate and one must integrate over it.
Otherwise, everything proceeds identically.
Unlike the periodic boundaries used for the 1D wave equation, the 2D wave equation is simulated in a box with Dirichlet boundary conditions, \(q=0\).
On boundary nodes, \(q\) is fixed and only the spatial derivatives are optimized, using the nearest complete patch; this leaves \(q=0\) as a hard constraint.

\Cref{fig:wave_2d_300s} compares ELM against a Fourier Neural Operator (FNO, \citealp{li2021fourier}) and PDE-Refiner \citep{lippe2023pderefiner} over \(1{,}000\)\,s.
The center-node trace (\cref{fig:wave_2d_center}) shows that ELM tracks the analytical ground truth with only minor phase drift and energy loss, while both FNO and PDE-Refiner have identifiable drift within the first \(20\)\,s.
The energy drift shown in \cref{fig:wave_2d_energy} tells a similar story.
While PDE-Refiner fails more gracefully than FNO, it still devolves into a non-physical global oscillation seen in \cref{fig:wave_2d_snapshots_body}.
ELM achieves better overall accuracy and energy conservation with \(78{,}000\) times fewer model parameters than PDE-Refiner.
The L\(^2\) error is provided in \cref{app:2d_wave}.  

To the best of our knowledge, this is the first work to simulate a PDE with two spatial dimensions using a learned Lagrangian density over long time-horizons.

\begin{figure}[t]
    \centering
    \begin{subfigure}[t]{0.48\linewidth}
        \centering
        \includegraphics[width=\linewidth]{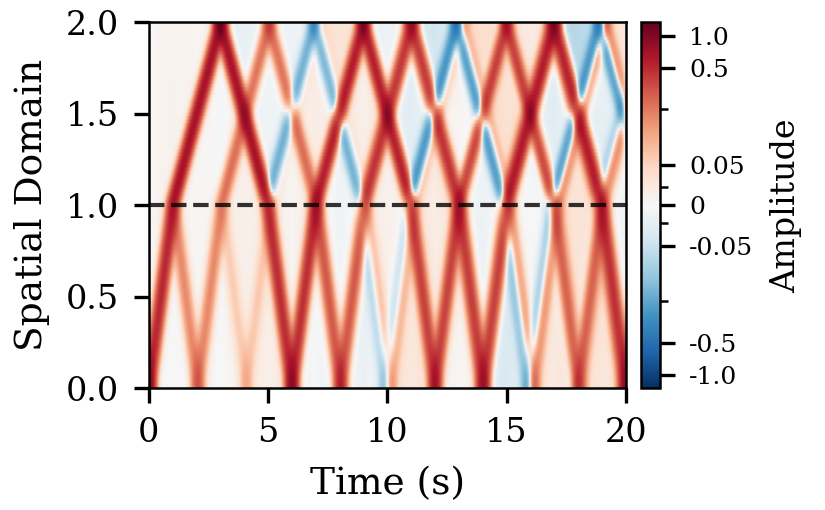}
        \caption{Spatio-temporal heatmap of a wave traversing a sigmoid impedance interface at \(x = 1\) (dashed line).}
        \label{fig:wave_x_heatmap}
    \end{subfigure}
    \hfill
    \begin{subfigure}[t]{0.48\linewidth}
        \centering
        \includegraphics[width=\linewidth]{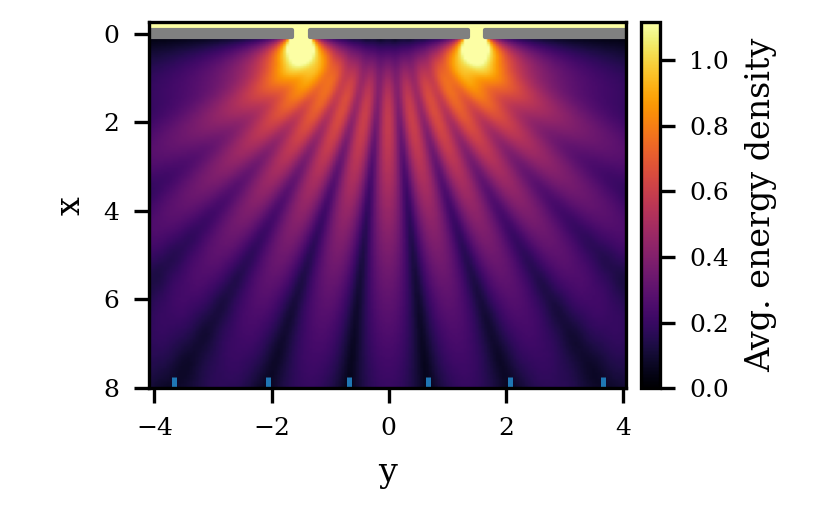}
        \caption{Time-averaged energy density for a driven wave hitting two slits, demonstrating interference.}
        \label{fig:double_slit}
    \end{subfigure}
    \caption{Demonstrations of (a) partial reflection and (b) diffraction and interference. 
    Both panels use models trained only on homogeneous data with simple boundary conditions. 
    Boundary effects are realized purely through the spatial variation of the Lagrangian or constraints on degrees of freedom.
    In Figure (b), the Fraunhofer minima for the fringes in this setup are shown as blue tick marks.}
    \label{fig:boundary_effects}
\end{figure}

\subsection{Emergent Boundary Effects}\label{sec:exp_boundary}

Both PDE results used identical boundary conditions in training as in inference.
Since learned continuous Lagrangians are pointwise, models compose and transfer to new domains without retraining. 

We first demonstrate this in 1D.
Suppose we have two learned Lagrangians independently trained on separate homogeneous media with different wave speeds and we want to explore what happens when they are abutted.
Here, we represent the transition as a sharp sigmoid blend between the two Lagrangians at \(x=1\).
\Cref{fig:wave_x_heatmap} shows the result: the wave partially reflects and partially transmits at the interface, exactly as waves behave when they hit an impedance barrier.
Compared to the analytical solution with a step boundary, the learned model accumulated only \({\sim}1\%\) energy drift over \(20\)\,s, with stable behavior at the interface.
The domain edges use homogeneous Neumann boundary conditions.

The same flexibility extends to 2D.
\Cref{fig:double_slit} shows the mean energy over time when a driven wave hits a barrier with two slits, producing interference fringes whose \(y\)-positions are within \(0.083\) (\({\sim}6\%\) of the fringe spacing) of those predicted by theory.
This example has three types of boundary conditions: a sinusoidally driven input on one edge, an internal wall with Dirichlet boundary conditions, and absorbing boundaries on all others.
Unlike the prior results, this configuration has driven and absorbing edges, so results demonstrate steady-state behavior.
This uses the identical model from \cref{sec:exp_2d}; barriers, slits, and driven boundaries are enforced purely through constraints on the degrees of freedom (see \cref{app:bcs}).
The operator baselines of \cref{sec:exp_2d} have no comparable method for transferring to this domain.

%% file: conclusion.tex
\section{Discussion and Conclusion}\label{sec:conclusion}

We have presented Euler--Lagrange Minimization (ELM), a novel, structure-preserving integrator that uses learned Lagrangians for forward simulation of ODE and PDE systems.  
Conceptually, ELM separates \textit{what} a system's physics are (the Lagrangian density) from \textit{how} to simulate them (the variational integrator).  
ELM achieves this through a construction at three scales: the local squared Euler--Lagrange residual which measures deviation from the learned dynamics, a patch error which describes the error induced by a finite basis over known states, and a minimization of this error via damped Newton iteration.
When the residual reaches zero, the integrator is symplectic for Lagrangians with a constant mass matrix, and symplectic up to quadrature error otherwise.
When it does not, the integrator remains structure-preserving, maintaining orders of magnitude less energy drift over long rollouts compared to existing methods.
For PDE systems in this work, we oversample the quadrature, trading exact symplecticity for the stability of the patch solver.

Our results show that ELM achieves energy conservation comparable to the symplectic GLRK integrator for the chaotic double pendulum.  
On 1D waves, ELM outperforms LNN at similar computational cost and, unlike LNN, scales accuracy with resolution.
ELM also constitutes the first simulation of a 2D PDE using a learned Lagrangian density, outperforming operator baselines with orders of magnitude fewer parameters.
More broadly, ELM avoids baking discretization into the learned model, so accuracy and compute can be traded at inference time. 

A further consequence of pairing a learned continuous Lagrangian with local minimization is that models compose and transfer across domains without retraining.
Spatial variation in the Lagrangian produces physically correct reflection and transmission at impedance interfaces; constraints on degrees of freedom realize Dirichlet, Neumann, driven, and absorbing boundaries through the same integrator.
The identical model used for the homogeneous 2D wave produces double-slit interference when given a barrier and a driven edge.
Operator-based methods, which fuse dynamics with discretization, cannot transfer this way.
By not leaving the integrator as an implementation detail, a single trained model behaves as a physics primitive.

ELM's strength lies in the regimes of low data availability, domain transfer and flexibility, and long rollouts.
Generating new data when the domain changes is a non-trivial task: operator methods such as FNO and PDE-Refiner may excel when the test domain matches the training domain, but have little to no adaptability when it changes, so new design requirements or geometries require the acquisition of an entirely new dataset.
ELM can perform this domain transfer (\cref{sec:exp_boundary}).
If a model is needed for applications such as a control loop, a neural surrogate may be trained from the resulting output.
Additionally, in many experimental settings (structural vibration, acoustics, etc.) dense grids are essentially unobtainable but sparse, pointwise data are simple to collect and suitable for ELM's pointwise training (\cref{app:sparse}).

Several limitations remain.
ELM is currently more expensive per time step than standard integrators, particularly for the ODE case where symplectic methods are already available.
The Jacobi sweep converges slowly for strongly-coupled systems, particularly in higher-dimensional PDEs where patches naturally overlap more.
Finally, ELM currently accepts only Lagrangian densities with up to first derivatives, excluding systems such as beam mechanics, and its guarantees do not yet cover strong derivative nonlinearities (\cref{app:nonlinear}).
None of these limitations require rethinking the local minimization framework to address.
Extending ELM to dissipative systems, higher-order field theories, and adaptive mesh refinement are natural directions for future work.

Taken together, the results of this paper argue that the integrator deserves the same care as the model to achieve flexible physics-learning systems.
ELM offers a starting point for learned models that can be deployed, constrained, and combined at inference time rather than retrained per scenario.

%% file: appendix.tex
\appendix
\crefalias{section}{appendix}
\crefalias{subsection}{appendix}

\section{Computational Cost Summary}\label{app:cost}

\input{computational_cost_table}

\paragraph{Newton iteration.}
The damped Newton iteration (specifically the repeated construction of the Hessian) is the main computational cost.
While nothing in the method inherently requires this choice of optimizer, the resulting quadratic convergence is highly beneficial.
Empirically, this is because the long-term behavior of the integrator relies on a good convergence of each step, and quadratic methods get closer to the true minimum than linear ones.
Additionally, since every step slightly changes the minimum for the adjacent nodes, the larger Newton steps allow us to converge in drastically fewer iterations.
The coupling also means that the Newton step must be damped slightly to prevent oscillations between iterations.

To provide numerical context: for the \(n{=}31\) 1D wave equation, if we use a simple gradient descent and tune the step size to the largest value that remains stable (smaller step sizes do not show better behavior), after \(25{,}000\) steps the resulting residual (\({\sim}10^{-3}\)) is still roughly an order of magnitude higher than the residual after three Newton steps (\({\sim}10^{-4}\)).

\paragraph{Cost relative to neural surrogates.}
While FNO and PDE-Refiner are presented for comparison, these methods are neural surrogates while ELM is not.
Instead of trying to provide a cost-efficient replacement for physical simulation, ELM learns a quantity that allows for the physical simulation of the system, and is best understood as a learned implicit solver.
Actually simulating the dynamics gives the domain transfer of \cref{sec:exp_boundary}, but comes with additional computational cost.
This cost should be balanced against the time and compute required to collect or generate large amounts of dense data and train a large model.

\FloatBarrier
\section{Deriving the Generating Function}\label{app:gen_func}
Our goal is to relate the differential of the discrete action \(\mathrm{d}S_d\) to the Euler--Lagrange residual \(R\).

\begin{lemma} \label{lem:dif_L}
    The differential of the Lagrangian \(\mathrm{d}\mathcal{L}\) can be expressed in terms of the residual \(R\) as
    \begin{equation*}
        \mathrm{d}\mathcal{L} = R \;\mathrm{d}q + \frac{\mathrm{d}}{\mathrm{d}t}\left( \frac{\partial \mathcal{L}}{\partial q_t} \;\mathrm{d}q \right).
    \end{equation*}
\end{lemma}

\begin{proof}
    We proceed via product rule, noting that \(\frac{\mathrm{d}}{\mathrm{d}t} \;\mathrm{d}q= \mathrm{d}q_t\):
    \begin{align*}
        R \;\mathrm{d}q + \frac{\mathrm{d}}{\mathrm{d}t}\left( \frac{\partial \mathcal{L}}{\partial q_t} \;\mathrm{d}q \right) &= \left(\frac{\partial \mathcal{L}}{\partial q} \;\mathrm{d}q - \frac{\mathrm{d}}{\mathrm{d}t}\frac{\partial \mathcal{L}}{\partial q_t} \;\mathrm{d}q\right) + \left( 
        \frac{\mathrm{d}}{\mathrm{d}t}\frac{\partial \mathcal{L}}{\partial q_t} \;\mathrm{d}q + \frac{\partial \mathcal{L}}{\partial q_t} \frac{\mathrm{d}}{\mathrm{d}t}\mathrm{d}q\right) \\
        &= \frac{\partial \mathcal{L}}{\partial q} \;\mathrm{d}q + \frac{\partial \mathcal{L}}{\partial q_t} \mathrm{d}q_t \\
        &= \mathrm{d}\mathcal{L} \hfill\qedhere 
    \end{align*}
\end{proof}

We now proceed in deriving the generating function form shown in \cref{eq:gen_func}:
\begin{align*}
    \mathrm{d}S_d &= \sum_{i=1}^{n_\omega} \omega_i \;\mathrm{d}\mathcal{L} \\
    &= \sum \omega_i R \;\mathrm{d}q + \sum \omega_i \frac{\mathrm{d}}{\mathrm{d}t}\left( \frac{\partial \mathcal{L}}{\partial q_t} \;\mathrm{d}q \right) \\
    &= \sum \omega_i R \;\mathrm{d}q + \int_{t_0}^{t_1} \frac{\mathrm{d}}{\mathrm{d}t}\left( \frac{\partial \mathcal{L}}{\partial q_t} \;\mathrm{d}q \right) \;\mathrm{d} t + \mathcal{O}\left(\Delta t^{2n_\omega+1}\right) \\
    &= \sum \omega_i R \;\mathrm{d}q  + \left.\frac{\partial \mathcal{L}}{\partial q_t} \;\mathrm{d}q\;\right|_{t_0}^{t_1} + \mathcal{O}\left(\Delta t^{2n_\omega+1}\right) \\
    &= \sum \omega_i R \;\mathrm{d}q + p(t_1)\,\mathrm{d}q(t_1) - p(t_0)\,\mathrm{d}q(t_0) + \mathcal{O}\left(\Delta t^{2n_\omega+1}\right).
\end{align*}
Here \(R\) and \(\mathrm{d}q\) inside the sums are evaluated on \(\hat q\) at the quadrature points \(\widetilde t_i\), as in \cref{eq:gen_func}.
Transitioning between each line, this is the result of the following steps:
\begin{enumerate}
    \item Applying \cref{lem:dif_L}.
    \item Using the standard approximation error of GL quadrature.
    \item Applying the fundamental theorem of calculus.
    \item Using the Legendre transformation to obtain the momenta \(p(t_0)\) and \(p(t_1)\).
\end{enumerate}

As discussed in \cref{sec:near_symp}, when \(R \rightarrow 0\) and the step size \(\Delta t\) is sufficiently small, this means that the update induced by ELM can be represented by a type-1 generating function to within the \(\mathcal{O}(\Delta t^{2n_\omega+1})\) term.
Thus, it is symplectic to within quadrature error.
This error term is the standard Gauss--Legendre approximation error \citep{stoer2002introduction}.
For \(n_\omega=2\), the resulting defect is only \(\frac{\Delta t^{5}}{4320}\) times a bound on the fourth time-derivative of the underlying integrand, for sufficiently smooth \(\mathcal{L}\).
For Lagrangians with a constant mass matrix, \cref{app:wronskian} shows that this approximation error vanishes from the symplectic defect.

\FloatBarrier
\section{\texorpdfstring{Exact Symplecticity for Constant-Mass Lagrangians}{Exact Symplecticity for Constant-Mass Lagrangians}}\label{app:wronskian}
This appendix proves \cref{prop:mechanical} and gives numerical evidence for \cref{conj:pde}.

\paragraph{Proof of \cref{prop:mechanical}.}
For a time step \(\Delta t\), denote the tangent map of the ELM step as \(A\).
For a Lagrangian \(\mathcal{L} = \tfrac12 q_t^\top M q_t - V(q)\), we are trying to show that \(A\) preserves the canonical two-form \(\Omega = M\,\mathrm{d}q \wedge \mathrm{d}q_t\).

Let \(u(t), w(t)\) be the variations of \(\hat q\) along two directions of the state \((q, q_t)\) and let \(u_0 = (u(0), u_t(0))\) and \(u_1 = (u(\Delta t), u_t(\Delta t)) = A u_0\).
Since \(\hat{q}\) is a cubic depending linearly on its endpoint data, \(u\) and \(w\) are likewise cubic functions.
As the solve has converged, we can differentiate the residual \(R\) with respect to this variation to obtain the variational conditions
\begin{equation*}
    M u_{tt}(\widetilde t_i) = -\nabla^2 V(\hat q(\widetilde t_i))\, u(\widetilde t_i),
\end{equation*}
and likewise for \(w\). We now define the Wronskian as
\begin{equation*}
    W(t) = u^\top M w_t - w^\top M u_t.
\end{equation*}
Note that this is exactly the form of the two-form \(\Omega\).
If we can show that \(W(0) = W(\Delta t)\), then the map is symplectic.
Taking the derivative,
\begin{equation*}
    W_t(t) = u^\top M w_{tt} - w^\top M u_{tt},
\end{equation*}
we find that since \(M\) is symmetric, higher-order terms cancel and \(W_t\) is a cubic function of \(t\).
Substituting the variational conditions above, \(W_t(\widetilde t_i) = 0\) by the symmetry of \(\nabla^2 V\).
Combining everything together, we find that
\begin{equation*}
    W(\Delta t) - W(0) = \int_0^{\Delta t} W_t(t)\,\mathrm{d}t = \sum_{i=1}^{2} \omega_i W_t(\widetilde t_i) = 0
\end{equation*}
exactly, since two-point Gauss--Legendre quadrature of a cubic is exact.
As \(u\) and \(w\) are arbitrary, we conclude that the step induced by ELM is symplectic.

\paragraph{Numerical evidence for \cref{conj:pde}.}
When the step does not converge such that \(J = 0\), the symplectic defect can be expressed as the exterior derivative of \cref{eq:gen_func}:
\begin{equation*}
    \mathrm{d}p_1\wedge\mathrm{d}q_1 - \mathrm{d}p_0\wedge\mathrm{d}q_0 = -\sum_{i=1}^{n_\omega} \omega_i\,\mathrm{d}R[\hat q\,|\,\widetilde t_i]\wedge\mathrm{d}\hat q(\widetilde t_i) - \mathrm{d}\,\mathcal{O}(\Delta t^{2n_\omega+1}).
\end{equation*}
This equation only examines the symplecticity of one time step to the next, the update induced by ELM, not the mechanics of ELM.
As a result, it applies without modification to PDEs when summing over the domain.
It provides two independent ways of measuring the defect, both computed directly via automatic differentiation: (a) the true defect through the tangent map \(A\), \(A^\top \Omega A - \Omega\), and (b) the summation of \(\mathrm{d}R\wedge\mathrm{d}\hat q\).
Using the agreement of these two independent measurements as numerical evidence, we evaluate \cref{conj:pde} numerically.
These measurements are computed in float64 rather than the float32 used elsewhere in the paper.

For part (i), we first validate that \(J\) reaches \(0\) in the edge case \(q \in \mathrm{span}\,\mathcal{F}\) noted in \cref{sec:near_symp}.
For any cubic polynomials \(f\) and \(g\), \(q(x, t) = f(x - ct) + g(x + ct)\) is an exact solution of the wave equation lying in \(\mathrm{span}\,\mathcal{F}\).
Pinning the boundary nodes to the exact solution (polynomials are not periodic), the converged step attains \(J \approx 10^{-28}\) and reproduces the exact flow to \({\sim}10^{-15}\).
Accordingly, the defect (a) is \({\sim}10^{-15}\) (accounting for the flux at the boundaries of the domain), matching (b) at \({\sim}10^{-16}\).

For part (ii), for both the linear wave equation and the sine-Gordon equation (\cref{app:nonlinear}) in the configurations used in this paper, we roll out from the initial condition for 100 steps and then measure both (a) and (b) for the next step.
These two calculations, sharing no intermediate quantities, agree to \({\sim}10^{-16}\).
Rolling out the Born--Infeld equation in the same way, the two calculations differ by \({\sim}0.05\%\).
When \(J>0\), at least \(99.9\%\) of the symplectic defect for this equation is therefore attributable to the residual \(R\), the quantity minimized by the integrator, even for this nonlinear system.

\FloatBarrier
\section{Training Procedures}\label{app:training}

The model and all surrounding code were implemented in JAX \citep{jax2018github} and with the Equinox library \citep{kidger2021equinox}.
Training and inference were performed on a single machine with an NVIDIA\textsuperscript{\textregistered}  5070 Ti.

The remainder of this appendix begins by describing the training and execution for ELM across all systems.
Descriptions of architecture and training data follow for each system examined in \cref{sec:results}.
These subsections also describe how comparison methods were trained, where appropriate.
Networks are named in line with the \textbf{Net} column of \cref{tab:cost}.
The remaining subsections provide some further discussion on the data requirements for training ELM.

\subsection{Training ELM}

The three ELM networks (\textbf{A}, \textbf{C}, \textbf{D}) share a single training procedure.
Each network minimizes the mean squared Euler--Lagrange residual (\(R^2\)) on a batch of pointwise samples, plus a small set of \emph{gauge constraints} that pin down the affine ambiguity of the Lagrangian.
For ODEs, for any function \(\mathcal{L}(q,q_t)\) which satisfies the Euler--Lagrange equations, the Lagrangian
\begin{equation*}
    \mathcal{L}^*(q,q_t)= A\mathcal{L}(q,q_t) + Bq_t +C
\end{equation*}
also satisfies the equations for any \(A,B,C\in \mathbb{R}\).
A similar statement holds for PDEs.

To counteract this ambiguity simply without affecting the solution space, a set of soft gauge constraints are added to the loss:
\begin{equation}
    L = R^2 + \left(\frac{\partial^2\mathcal{L}(0)}{\partial q_t^2} - 1\right)^2 + \left(\frac{\partial\mathcal{L}(0)}{\partial q_t}\right)^2 +\mathcal{L}(0)^2 \label{eq:loss}
\end{equation}
Combined, these gauge constraints seek to fix \((A,B,C) = (1,0,0)\).
While \(R^2\) is evaluated across the batch, these gauge constraints only fix \(\mathcal{L}\) at the origin. 
While this is sufficient for the wave equation, more complicated Lagrangians such as the double pendulum require additional assistance.

The first step of any ELM rollout requires the Hermite degrees of freedom at \(t = 0\): 
\begin{enumerate}
    \item \((q, q_t)\) for the double pendulum.
    \item \((q, q_t, q_x, q_{xt})\) for the 1D wave.
    \item \((q, q_t, q_x, q_y, q_{xt}, q_{yt}, q_{xy}, q_{xyt})\) for the 2D wave.
\end{enumerate}
These can either be computed analytically from the closed-form initial condition or estimated from data, the same way a finite-difference solver seeds its initial half-step. 
No multi-frame warm-up is required.

\subsection{Double-pendulum network (A)}\label{app:training_a}

Using the loss described by \citet{cranmer2020lagrangian}, a generic MLP can learn the double-pendulum Lagrangian directly from \((q, q_t, q_{tt})\) data. 
Let \(M(q) = \frac{\partial^2 \mathcal{L}}{ \partial q_t^2}\) be the mass matrix of the network.
Then, the LNN acceleration can be computed as
\begin{equation*}
    q_{tt} = M(q)^{-1}\left(\frac{\partial \mathcal{L}}{ \partial q} - \frac{\partial^2 \mathcal{L}}{\partial q \partial q_t}q_t\right),
\end{equation*}
whose difference against the acceleration in data is minimized during training.
By taking the inverse of the mass matrix, this loss implicitly regularizes the network away from eigenvalues close to zero.
Without this regularization, the gauge terms in \cref{eq:loss} are insufficient to prevent a collapse of the mass matrix in this system for a generic network.

While one solution to this problem might be to either use the original LNN loss or apply more recent work on stabilizing LNN training \citep{hamzaogullari2026stabilized}, to keep the training procedure simple and consistent we instead adopt a DeLaN inspired architecture for the double-pendulum \emph{only}.
The new model is defined as
\begin{equation*}
    \mathcal{L}(q,q_t) = \tfrac{1}{2} q_t^T M(q) q_t - V(q),
\end{equation*}
where \(M(q)\) is produced from a Cholesky factor and \(V(q)\) is a separate scalar head.
The network reads \((\sin q_0, \cos q_0, \sin q_1, \cos q_1, q_{t,0}, q_{t,1})\) through three hidden layers of width 200 with Softplus activations. 
It is trained for \(200{,}000\) steps on \(20{,}000\) samples of \((q, q_t, q_{tt})\) generated from the analytical double-pendulum Lagrangian, with angles drawn uniformly from \([-\pi, \pi]\) and angular velocities up to \(q_t^{\max} = 8\).

The imposition of structure here is accepted for the following reasons:
\begin{itemize}
    \item We know that a generic MLP can learn a double pendulum and that the resulting model can be used here without modification.
    \item All comparisons in \cref{sec:exp_ode} use the same network, so ELM does not gain any advantage.
    \item As ELM represents a structure-preserving integrator, and since any learned Lagrangian represents some set of dynamics, the model error is separate from the integration error and does not compound.
\end{itemize}

\subsection{1D wave networks (B, C)}\label{app:training_bc}
Networks \textbf{B} and \textbf{C} share hidden-layer width and differ only in the number of inputs. 
\textbf{B} tries to learn the spatial discretization, while \textbf{C} is a strictly pointwise function.

\paragraph{LNN baseline (B).} 
The architecture from \citet{cranmer2020lagrangian} consists of an MLP with three hidden layers of width 400 and Softplus activations, taking the 6-tuple \((q(t,x_{i-1}), q(t,x_i), q(t,x_{i+1}), q_t(t,x_{i-1}), q_t(t,x_i), q_t(t,x_{i+1}))\) from a 3-node spatial stencil. 
This model was trained on 1D wave trajectories (\(q_{tt} = 0.05\,q_{xx}\) with periodic boundary conditions) using the released code and training procedure, with the best model from 200 random seeds returned.
LNN training is unstable enough that only \({\sim}5\)--\(10\%\) of seeds converge to a test loss below 1.5.

\paragraph{ELM (C).} 
Training samples are drawn from random superpositions of plane waves on the relevant domain (5 modes, wavenumbers are uniform in \([-k_\text{bound}, k_\text{bound}]\), frequencies are set by the linear dispersion relation \(\omega = c|k|\)). 
All derivatives the loss touches are computed analytically from the closed-form superposition. 
Training was conducted with Adam with a cosine one-cycle schedule (learning rate \(10^{-3} \to 0.2\)).

The network is an MLP with three hidden layers of width 400 and Softplus activations, taking the \emph{pointwise} 3-tuple \((q, q_t, q_x)\) at a single space-time point. 
The model was trained for \(50{,}000\) steps on \(5{,}000\) plane-wave samples (\(k_\text{bound} = 4\), \(c^2 = 0.05\)).
To characterize seed variability, we retrained this model on ten unselected random seeds (0--9) under the identical recipe, regenerating the training data per seed.
All ten seeds converge and all ten full-horizon rollouts are stable at every resolution, with the energy errors reported in \cref{tab:cost}.
Even the worst unselected seed at \(n{=}31\) (\(8.8\%\)) conserves energy \(78{\times}\) better than the best-of-200 LNN model at \(\Delta t{=}0.4\).

\subsection{2D wave networks (D, E, F)}\label{app:training_def}

\paragraph{ELM (D).}
Training samples are drawn identically to the 1D wave case, only this time using a superposition of 2D waves.
The network is an MLP with two hidden layers of width 40 and Softplus activations, taking the \emph{pointwise} 4-tuple \((q, q_t, q_x, q_y)\) at a single space-time point. 
The model was trained for \(100{,}000\) steps on \(10{,}000\) plane-wave samples (\(k_\text{bound} = 6\), \(c^2 = 1\)).

\paragraph{PDE-Refiner (E).} 
A U-Net conditioned on a diffusion-step embedding, trained as a denoiser with 8 denoising steps and a 4-frame input history. 
The training data is 400 eigenmode-decomposed trajectories of the 2D wave on a \(64 \times 64\) grid with \(\Delta t = 0.02\)\,s.
Half of this training data is a random superposition of modes, half is simulated initial Gaussian pulses.
The model was trained for 400 epochs.

\paragraph{FNO (F).} 
A standard Fourier Neural Operator, trained as a one-step predictor with a 10-frame input history. 
The training data is the same 400 eigenmode trajectories on \(64 \times 64\) at \(\Delta t = 0.02\)\,s, trained for 500 epochs.

\paragraph{Comparison of data requirements.} 
The disparity between dataset sizes is worth quantifying. 
Network \textbf{D} sees \(10^{4}\) pointwise samples, each a 4-tuple of scalars, totalling \(4 \times 10^{4}\) scalar values seen during training. 
Networks \textbf{E} and \textbf{F} each see 400 trajectories of length \(T\) on a \(64 \times 64\) grid at \(\Delta t_\text{save} = 0.02\)\,s, sliding the input/target window across each trajectory. 
With even a modest trajectory length of \(T=1\) this is on the order of \(8 \times 10^{8}\) space-time scalar values per epoch, repeated for hundreds of epochs. 
ELM is trained on roughly five orders of magnitude less raw data, on a network up to five orders of magnitude smaller, and still produces the greater long time-horizon stability seen in \cref{fig:wave_2d_300s}.

\subsection{The cost of pointwise derivatives}\label{app:sparse}

ELM's training data and initial state require pointwise field derivatives, which may seem more demanding than the dense regular-grid snapshots that autoregressive baselines (\textbf{E}, \textbf{F}) consume. 
In practice, it is less so.
ELM's loss is strictly pointwise: training samples need not lie on a regular grid, share a coordinate frame, or be uniformly spaced.
In principle, six sensors placed near each other, sampled over sufficient time, produce all the data needed to train a 2D wave model for ELM.
This is the minimum needed to recover all necessary derivatives by finite differences.
The autoregressive baselines, by contrast, are tied to a dense mesh for both training and rollout and require a multi-frame input history (4 frames for PDE-Refiner, 10 for FNO).

To test this, we train a model for the 1D wave equation using only position data recorded at three closely spaced points, with all derivatives estimated by finite differences and the recipe otherwise identical to network \textbf{C}.
Using the \(n{=}31\) configuration of \cref{tab:cost}, the final error in energy at \(10{,}000\)\,s is \(2.94\%\), within the seed variability of network \textbf{C}.

\section{Extended 1D Wave Comparison}\label{app:heatmaps}

\Cref{fig:wave_l2,fig:heatmaps_all} extend \cref{sec:exp_wave} with the relative \(L^2\) error and full spatio-temporal heatmaps over the same \(10{,}000\)\,s problem (\(q_{tt} = 0.05\,q_{xx}\), periodic boundary conditions, Gaussian initial condition).

In addition to the LNN and ELM-learned configurations of \cref{fig:wave_combined}, both figures include two compute-matched symplectic baselines on the same \(n{=}51\), \(\Delta t{=}0.1\) discretization: GLRK and ELM-ODE. 
Both methods share the same fourth-order finite differences to compute spatial derivatives and both still use the same underlying learned Lagrangian.
To use GLRK, an LNN-style estimate for the second time derivative is used.
ELM-ODE applies the ODE form of ELM (\cref{sec:near_symp}) to the same spatially discretized system.

ELM at all three resolutions mostly preserves wave structure for the full duration, with higher resolution producing sharper wavefronts. 
LNN at both time steps shows progressive amplitude decay and phase distortion, consistent with the energy drift reported in \cref{fig:wave_combined}. 
GLRK and ELM-ODE are virtually identical on the heatmap and \(L^2\) plot, once again validating the symplecticity of ELM on ODEs.
These methods also show that much of the distortion exhibited by ELM on the various configs is the inherent phase error in the learned Lagrangian.
On short time scales, the joint space-time optimization of ELM achieves lower \(L^2\) error than the space-discretized symplectic integrators.
This suggests that minimizing on space-time patches avoids artifacts caused by fixed discretization.

\begin{figure}[h]
    \centering
    \includegraphics[width=0.7\linewidth]{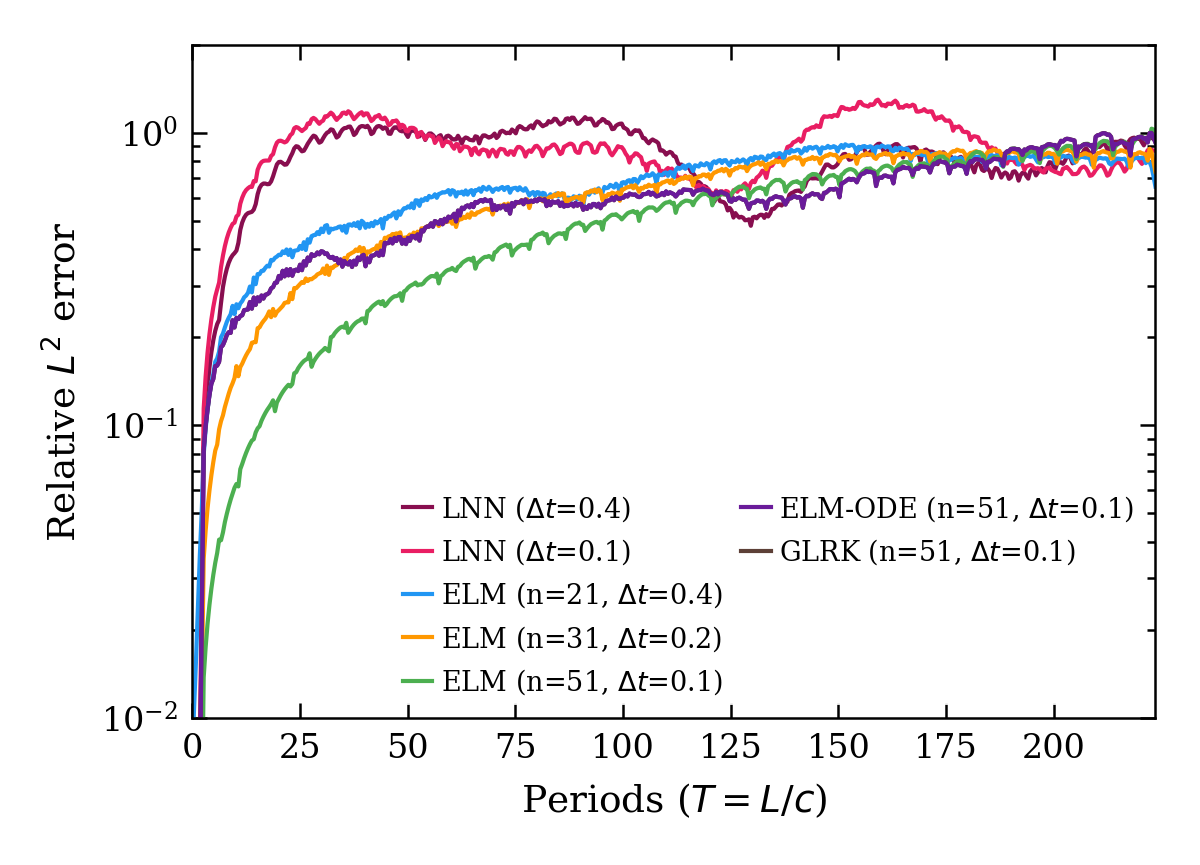}
    \caption{Relative \(L^2\) error vs.\ exact solution over the full \(10{,}000\)\,s (\(224\) periods) for the methods of \cref{fig:wave_combined}, plus the GLRK and ELM-ODE symplectic baselines.}
    \label{fig:wave_l2}
\end{figure}

\begin{figure}[h]
    \centering
    \includegraphics[width=0.8\textwidth]{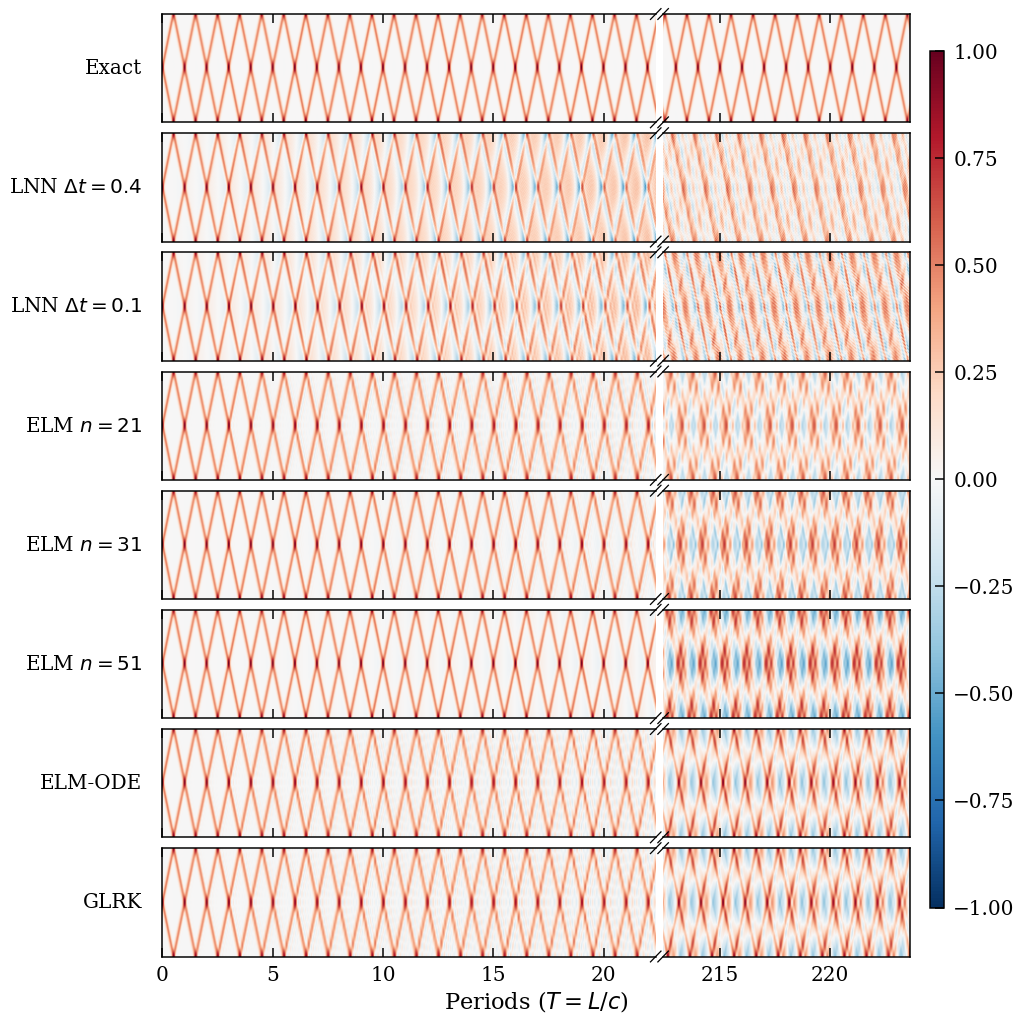}
    \caption{Spatio-temporal heatmaps \(q(x,t)\) over \(10{,}000\)\,s (\(224\) periods) extending \cref{fig:wave_heatmaps}. 
    The figure shows the first \({\sim}20\) periods and the last \({\sim}10\). Parameters for each method are listed in \cref{app:cost}.}
    \label{fig:heatmaps_all}
\end{figure}

\FloatBarrier
\section{\texorpdfstring{Nonlinear Conservative PDEs}{Nonlinear Conservative PDEs}}\label{app:nonlinear}
The applicability of ELM to nonlinear conservative PDEs depends on what type of nonlinearity we are including.
If we restrict ourselves to Lagrangians with a nonlinear potential, ELM can be applied without modification to the method.
Consider the sine-Gordon equation:
\begin{equation*}
    \mathcal{L} = \tfrac{1}{2}q_t^2 - \tfrac{1}{2}q_x^2 - (1 - \cos q).
\end{equation*}
We use the same training recipe as network \textbf{C} (\cref{app:training_bc}) on data generated from this equation, and run with the same integrator configuration as the 1D wave.

\paragraph{Breather.}
Given the ``breather'' solution and using \(\omega=0.6\) and \(t=0\) as initial conditions,
\begin{equation*}
    q(x, t) = 4\arctan\left[\frac{\eta}{\omega}\,\frac{\sin(\omega t)}{\cosh(\eta x)}\right], \qquad \eta = \sqrt{1-\omega^2},
\end{equation*}
over \(1{,}000\)\,s (\({\sim}95\) periods), using the same model, accuracy improves with resolution, just as with the linear wave model (\cref{fig:sg}, \cref{tab:sg}).

\paragraph{Kink--antikink collision.}
Similarly, we can simulate the kink--antikink collision with \(v=0.25\), initialized from the closed-form solution
\begin{equation*}
    q(x, t) = 4\arctan\left[\frac{\sinh(\gamma v t)}{v\cosh(\gamma x)}\right], \qquad \gamma = \frac{1}{\sqrt{1-v^2}},
\end{equation*}
at \(t = -12\)\,s, so that the kinks collide at \(t = 0\).
Using \(n=31\) and simulating for \(24\)\,s, the kinks pass through each other elastically (\cref{fig:sg_ka}); the error in energy across the collision is \(0.18\%\) with a maximum error of \(0.61\%\), the final relative \(L^2\) error is \(5.0\times10^{-2}\), and the post-collision kink centers are recovered to within \(0.21\).

\begin{table}[h]
\centering
\centering
\caption{Sine-Gordon results with the learned density. Energy is the maximum of \(|E_t - E_0|/E_0\) over the rollout; \(L^2\) is the relative error against the exact solution at the final time.}
\label{tab:sg}
\begin{tabular}{@{}lccrrr@{}}
\toprule
\textbf{Case} & \textbf{Nodes} & \textbf{\(\Delta t\)} & \textbf{\(T\)} & \textbf{\(|\Delta E/E_0|\)} & \textbf{\(L^2\)} \\
\midrule
Breather       & 21 & 0.190 & 1{,}000\,s & 0.67\%  & 0.97 \\
Breather       & 31 & 0.129 & 1{,}000\,s & 0.12\%  & 0.15 \\
Breather       & 51 & 0.078 & 1{,}000\,s & 0.036\% & 0.077 \\
Kink--antikink & 31 & 0.129 & 24\,s      & 0.61\%  & 0.050 \\
\bottomrule
\end{tabular}
\end{table}

\begin{figure}[h]
    \centering
    \includegraphics[width=\linewidth]{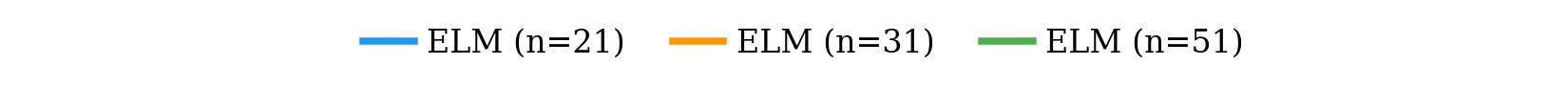}\\[-2pt]
    \begin{subfigure}[t]{0.48\linewidth}
        \includegraphics[width=\linewidth]{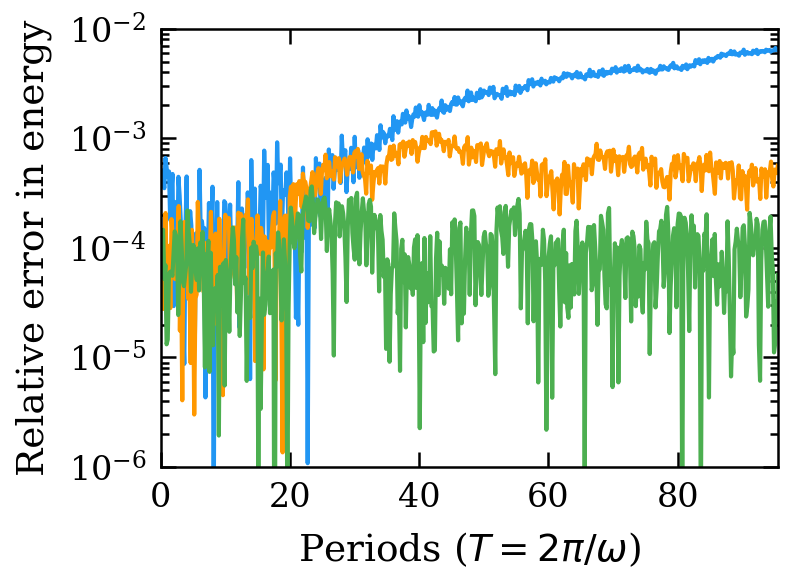}
        \caption{Breather: relative error in energy \(|E_t - E_0|/E_0\).}
    \end{subfigure}
    \hfill
    \begin{subfigure}[t]{0.48\linewidth}
        \includegraphics[width=\linewidth]{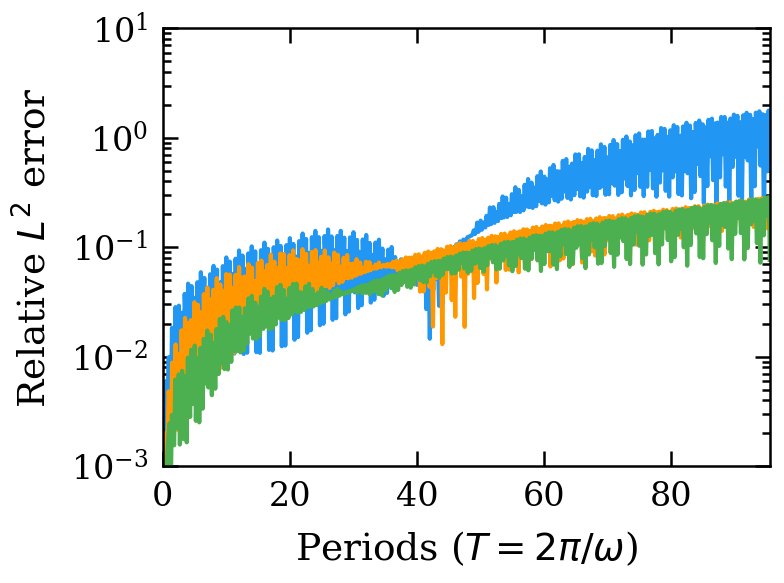}
        \caption{Breather: relative \(L^2\) error vs.\ the exact solution.}
    \end{subfigure}

    \vspace{0.5em}
    \begin{subfigure}[t]{0.48\linewidth}
        \includegraphics[width=\linewidth]{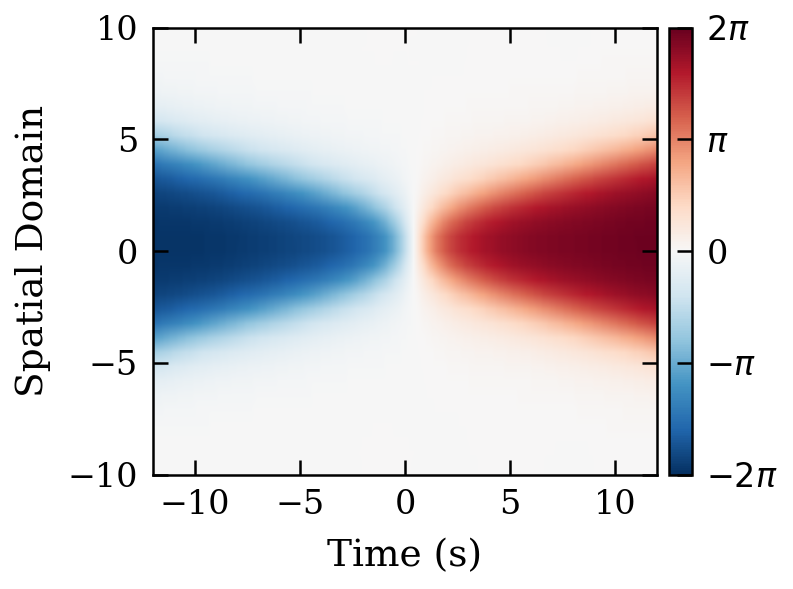}
        \caption{Kink--antikink \(q(x,t)\), \(n = 31\), \(v = 0.25\).}
        \label{fig:sg_ka}
    \end{subfigure}
    \hfill
    \begin{subfigure}[t]{0.48\linewidth}
        \includegraphics[width=\linewidth]{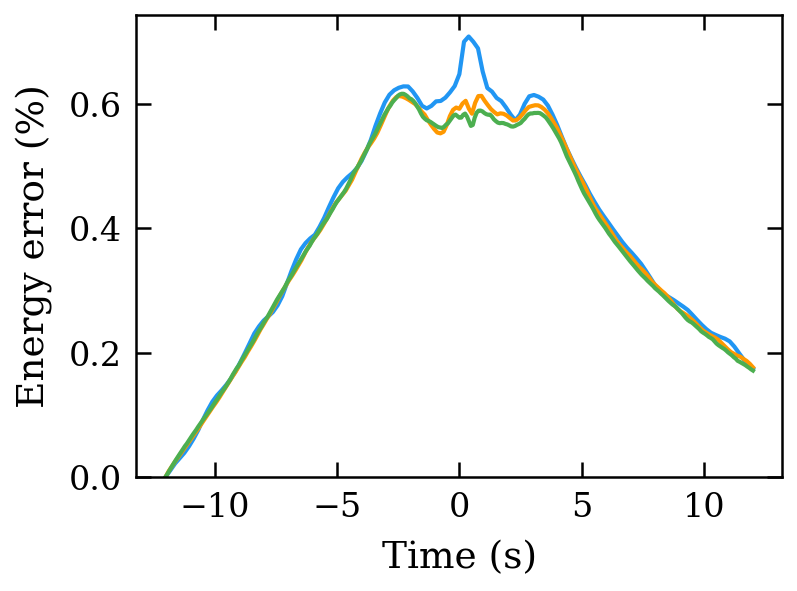}
        \caption{Kink--antikink: relative error in energy.}
    \end{subfigure}
    \caption{Sine-Gordon with a learned Lagrangian density. (a, b) Breather over \(1{,}000\)\,s at three resolutions. (c, d) Kink--antikink collision, which occurs at \(t = 0\); the kinks start and end at \(x \approx \pm 4.3\).}
    \label{fig:sg}
\end{figure}

\paragraph{Derivative nonlinearities.}
For more general classes of PDEs, the current formulation may accumulate drift more rapidly.
The Born--Infeld equation, \(\mathcal{L} = 1 - \sqrt{1 + q_x^2 - q_t^2}\), is quasilinear with strong nonlinearity in the derivatives, violating the assumptions made in \cref{prop:mechanical}.
For a model trained on this equation, the error in energy reaches \(2.2\%\) by \(t = 200\)\,s at \(n = 51\) (\(\Delta t = 0.01\), \(n_r = 6\)).
No integrator, however, can be considered suitable for every problem.
Where it applies, ELM retains the benefits demonstrated throughout the paper: no fixed discretization, easily implemented boundary conditions, and smaller data requirements.

\FloatBarrier
\section{Extended 2D Wave Comparison}\label{app:2d_wave}

\Cref{fig:wave_2d_l2} extends \cref{sec:exp_2d} with the relative \(L^2\) error against the eigenmode ground truth over the full \(1{,}000\)\,s rollout.
The qualitative result matches \cref{fig:wave_2d_energy}.
ELM is the only method that stays below \(10^{0}\); PDE-Refiner saturates near \(10^{0}\), and FNO diverges at \(t \approx 22\)\,s.

\begin{figure}[h]
    \centering
    \includegraphics[width=0.7\linewidth]{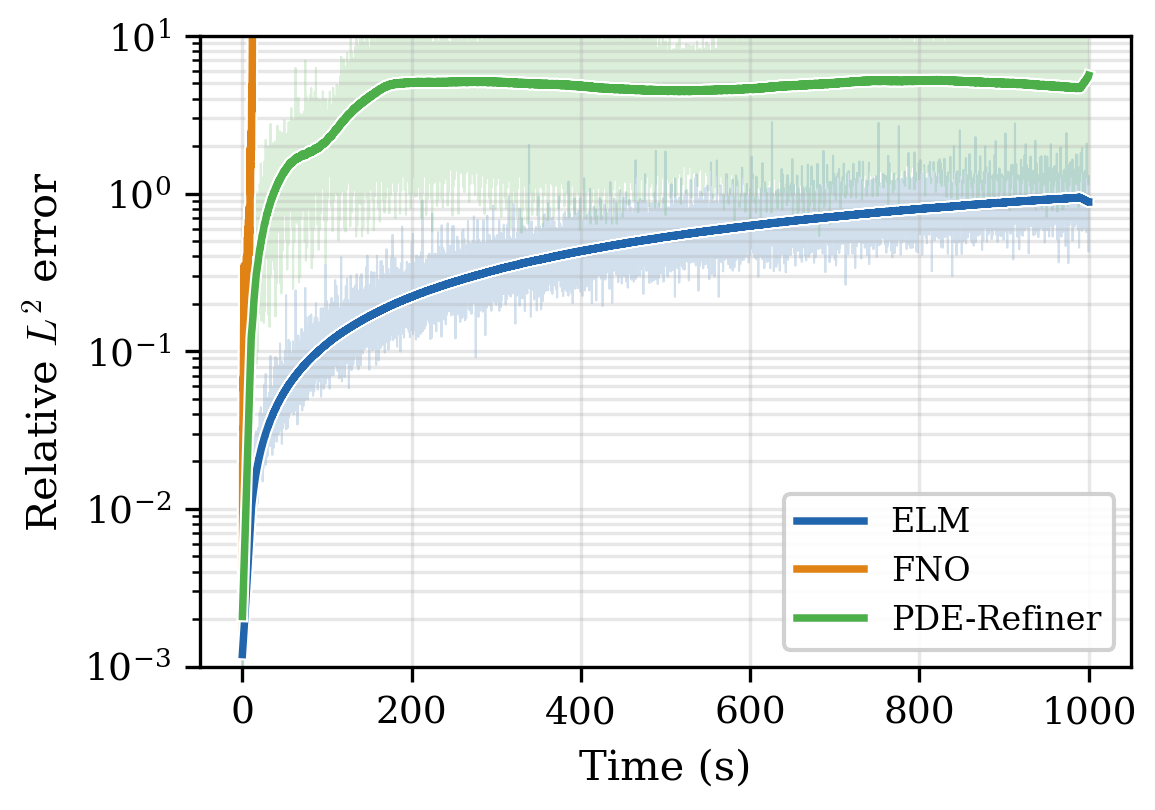}
    \caption{Relative \(L^2\) error vs.\ eigenmode ground truth for the 2D wave over \(1{,}000\)\,s for the
    same system and methods as \cref{fig:wave_2d_300s}.}
    \label{fig:wave_2d_l2}
\end{figure}

\FloatBarrier
\section{Boundary Condition Implementations}\label{app:bcs}

The experiments in \cref{sec:results} use five boundary conditions across the 1D and 2D wave equations.
\Cref{tab:bcs} lists the mathematical condition imposed at each boundary node and the section in which the boundary condition appears.
In ELM, every boundary condition is realized by constraining a subset of the per-node variables before the Jacobi sweep; no boundary condition requires a change to the loss or to the integrator core.

We implement the absorbing boundary as a Mur boundary condition \citep{mur1981absorbing}.
This boundary requires an estimate of the wave speed in the medium.
For a homogeneous medium, the wave speed can be estimated from the learned model as
\begin{equation*}
    \hat{c}^{2} = -\left.\frac{\partial^{2}\mathcal{L}}{\partial q_x^{2}}\middle/\frac{\partial^{2}\mathcal{L}}{\partial q_t^{2}}\right|_{0}.
\end{equation*}
As the boundary condition is used only to prevent boundary reflections from distorting the qualitative interference pattern, this requirement is acceptable.

\begin{table}[h]
\centering
\caption{Boundary conditions used in the experiments.
For the periodic boundary, \(q_\star\) indicates that all optimized variables are shared.
For the Mur boundary, \(s = \pm 1\) denotes the outward-normal sign (\(+1\) for the right/top edges, \(-1\) for left/bottom).}
\label{tab:bcs}
\small
\begin{tabular}{@{}lll@{}}
\toprule
\textbf{Name} & \textbf{Condition} & \textbf{Used in} \\
\midrule
Periodic     & \(q_\star(t, x_0)=q_\star(t, x_N)\)   & \cref{sec:exp_wave} (1D wave) \\
Dirichlet    & \(q=0\)                   & \cref{sec:exp_2d}, slit barrier in \cref{sec:exp_boundary} \\
Neumann      & \(q_x=0\)                 & \cref{sec:exp_boundary} (1D impedance edges) \\
Driven       & \(q=A\sin(\omega t)\)     & \cref{sec:exp_boundary} (driven edge of double-slit) \\
Mur          & \(q_t + s\,c\,q_x=0\)     & \cref{sec:exp_boundary} (other three edges of double-slit) \\
\bottomrule
\end{tabular}
\end{table}

%% file: computational_cost_table.tex
\begin{table*}[ht]
\centering
\caption{Configuration, computational cost, and maximum relative energy error for each method in \cref{sec:results}. 
Methods sharing a letter in the \textbf{Net} column use the same architecture and the same trained weights.
More details on architecture and training can be found in \cref{app:training}.
Computational cost excludes JIT compilation. 
The double pendulum experiments run on CPU (float64) while wave experiments run on GPU (float32).
Entries reported as mean \(\pm\) standard deviation are over ten unselected seeds (0--9), each retraining the network with the same recipe on training data regenerated per seed; costs are for the original models shown in the figures.
}
\label{tab:cost}
\small
\setlength{\tabcolsep}{4pt}
\begin{tabular}{@{}lccccrrrr@{}}
\toprule
\textbf{Method} & \textbf{Net} & \textbf{Params} & \textbf{Nodes} & \textbf{\(\Delta t\)} & \textbf{Steps} & \textbf{ms/step} & \textbf{Total} & \textbf{\(|\Delta E/E_0|\)} \\
\midrule
\multicolumn{9}{@{}l}{\textit{Double pendulum (100{,}000\,s; ELM \(n_r{=}3\), \(n_\omega=2\))}} \\
\addlinespace[2pt]
RK4               & \textbf{A} & \multirow{5}{*}{163{,}604} & --- & 0.02     & 5{,}000{,}000 & 0.18  & 924\,s       & 71.8 \(\pm\) 7.4\% \\
Dormand--Prince    & \textbf{A} &                            & --- & adaptive & ---           & ---   & 944\,s       & 16.91 \(\pm\) 0.23\% \\
Implicit Midpoint & \textbf{A} &                            & --- & 0.02     & 5{,}000{,}000 & 0.41  & 2{,}072\,s   & 6.1 \(\pm\) 2.1\% \\
GLRK              & \textbf{A} &                            & --- & 0.02     & 5{,}000{,}000 & 0.76  & 3{,}815\,s   & 0.0267 \(\pm\) 0.0059\% \\
ELM               & \textbf{A} &                            & --- & 0.02     & 5{,}000{,}000 & 2.4   & 11{,}977\,s  & 0.039 \(\pm\) 0.020\% \\
\midrule
\multicolumn{9}{@{}l}{\textit{1D wave (periodic boundaries; 10{,}000\,s; ELM \(n_r=3\), \(n_\omega=3\text{ in time and }6\text{ in space}\))}} \\
\addlinespace[2pt]
LNN \(\Delta t{=}0.4\)  & \textbf{B} & \multirow{2}{*}{324{,}001} & 100 & 0.4 & 25{,}000  & 5.7 & 141\,s      & 680\%      \\
LNN \(\Delta t{=}0.1\)  & \textbf{B} &                            & 100 & 0.1 & 100{,}000 & 5.5 & 549\,s      & 1{,}680\%  \\
\addlinespace[2pt]
ELM \(n{=}21\)          & \textbf{C} & \multirow{3}{*}{322{,}801} & 21  & 0.4 & 25{,}000  & 9.4 & 235\,s      & 17.6 \(\pm\) 3.9\% \\
ELM \(n{=}31\)          & \textbf{C} &                            & 31  & 0.2 & 50{,}000  & 9.6 & 481\,s      & 4.1 \(\pm\) 2.4\% \\
ELM \(n{=}51\)          & \textbf{C} &                            & 51  & 0.1 & 100{,}000 & 15  & 1{,}500\,s  & 4.1 \(\pm\) 4.0\% \\
\midrule
\multicolumn{9}{@{}l}{\textit{2D wave (Dirichlet boundaries, 1{,}000\,s; ELM \(n_r=20\), \(n_\omega=3\text{ in time and }5 \times 5\text{ in space}\))}} \\
\addlinespace[2pt]
ELM         & \textbf{D} & 1{,}881      & \(31{\times}31\) & 0.02 & 50{,}000 & 968 & 48{,}400\,s & 19\%          \\
PDE-Refiner & \textbf{E} & \({\sim}\)146M & \(64{\times}64\) & 0.02 & 50{,}000 & 26  & 1{,}300\,s  & 271\%        \\
FNO         & \textbf{F} & 138{,}301    & \(64{\times}64\) & 0.02 & 50{,}000 & 2.9 & 145\,s      & diverges\textsuperscript{\textdagger} \\
\bottomrule
\end{tabular}
\vspace{4pt}

\noindent
\textsuperscript{\textdagger} NaN at \(t = 21.9\)\,s; energy exceeds \(10^{36}\) before divergence.
\end{table*}